\documentclass[lettersize,journal]{IEEEtran}
\usepackage[T1]{fontenc}
\usepackage{amsmath,amsfonts}
\usepackage{amsthm}
\usepackage{acronym}
\usepackage[linesnumbered,ruled,vlined]{algorithm2e}
\usepackage{array}
\usepackage[table]{xcolor}
\usepackage{textcomp}
\usepackage{stfloats}
\usepackage{subcaption}
\usepackage{multirow}
\usepackage{bbding}
\usepackage{hyperref}
\usepackage{tcolorbox}
\usepackage{url}
\usepackage{verbatim}
\usepackage{graphicx}
\usepackage{cite}
\SetKwInput{KwInput}{Input}               
\SetKwInput{KwOutput}{Output}

\newtheorem{assumption}{Assumption}
\newtheorem{corollary}{Corollary}
\newtheorem{definition}{Definition}
\newtheorem{lemma}{Lemma}
\newtheorem{problem}{Problem}
\newtheorem{proposition}{Proposition}
\newtheorem{remark}{Remark}
\newtheorem{theorem}{Theorem}

\acrodef{dp}[DP]{differential privacy}
\acrodef{dpfl}[DP-FL]{Differentially private federated learning}
\acrodef{fedavg}[FedAvg]{federated averaging}
\acrodef{fedsgd}[FedSGD]{federated stochastic gradient descent}
\acrodef{fl}[FL]{Federated learning}
\acrodef{iid}[IID]{independent and identically distributed}
\acrodef{ldp}[LDP]{local differential privacy}
\acrodef{non-iid}[non-IID]{not independent and identically distributed}
\acrodef{sgd}[SGD]{stochastic gradient descent}
\acrodef{cnn}[CNN]{convolutional neural network}

\let\oldfrac\frac
\renewcommand{\frac}[2]{%
	\mathchoice
	{\oldfrac{#1}{#2}}
	{#1/#2}
	{#1/#2}
	{#1/#2}
}

\begin{document}

\title{Tackling Privacy Heterogeneity in Differentially Private Federated Learning}

\author{Ruichen~Xu, Ying-Jun Angela Zhang, ~\IEEEmembership{Fellow, ~IEEE}, Jianwei Huang, ~\IEEEmembership{Fellow, ~IEEE} 
	\thanks{An earlier version of this paper was presented in part at the IEEE WiOpt 2023 \cite{10349854}.
		Ruichen Xu and Ying-Jun Angela Zhang are with the Department of
		Information Engineering, The Chinese University of Hong Kong, Hong Kong (e-mail: xr021@ie.cuhk.edu.hk; co-corresponding author, email: yjzhang@ie.cuhk.edu.hk).
		Jianwei Huang is with the School of Science and Engineering, Shenzhen Institute of Artificial Intelligence and Robotics for Society, Shenzhen Key Laboratory of Crowd Intelligence Empowered Low-Carbon Energy Network, and CSIJRI Joint Research Centre on Smart Energy Storage, The Chinese University of Hong Kong, Shenzhen, Guangdong, 518172, P.R. China (co-corresponding author, email: jianweihuang@cuhk.edu.cn).}}

\maketitle

\begin{abstract}
	Differentially private federated learning (DP-FL) enables clients to collaboratively train machine learning models while preserving the privacy of their local data. However, most existing DP-FL approaches assume that all clients share a uniform privacy budget, an assumption that does not hold in real-world scenarios where privacy requirements vary widely. This privacy heterogeneity poses a significant challenge: conventional client selection strategies, which typically rely on data quantity, cannot distinguish between clients providing high-quality updates and those introducing substantial noise due to strict privacy constraints. To address this gap, we present the first systematic study of privacy-aware client selection in DP-FL. We establish a theoretical foundation by deriving a convergence analysis that quantifies the impact of privacy heterogeneity on training error. Building on this analysis, we propose a privacy-aware client selection strategy, formulated as a convex optimization problem, that adaptively adjusts selection probabilities to minimize training error. 
	Extensive experiments on benchmark datasets demonstrate that our approach achieves up to a 10\% improvement in test accuracy on CIFAR-10 compared to existing baselines under heterogeneous privacy budgets. These results highlight the importance of incorporating privacy heterogeneity into client selection for practical and effective federated learning.
\end{abstract}


\section{Introduction}
\ac{fl} is an emerging distributed machine learning paradigm that enables clients—such as mobile devices, hospitals, or financial institutions to collaboratively train models while retaining their data locally. By exchanging only model updates with a central server, FL offers a decentralized approach that helps safeguard client privacy and reduce data transfer costs. However, sharing gradients or model updates alone does not guarantee privacy. Recent studies have demonstrated that adversaries can reconstruct sensitive information from shared gradients, exposing clients to inference attacks~\cite{zhu2019deep}.

To address these vulnerabilities, \ac{dp} has been widely adopted as a rigorous framework for privacy protection in FL. In \ac{dpfl}, clients inject random noise into their model updates before uploading them to the server, thereby providing quantifiable privacy guarantees~\cite{bonawitz2021federated}. The strength of privacy protection is controlled by each client's privacy budget: a smaller budget offers stronger privacy but requires more noise, which can degrade model utility. Thus, DP-FL systems must carefully manage the trade-off between privacy and learning performance.

This challenge is further compounded by the inherent heterogeneity among clients in practical FL deployments. Beyond the well-studied differences in data size and distribution, clients' contributions in DP-FL are fundamentally constrained by their individual privacy budgets. In real-world applications, privacy requirements can vary significantly: some clients may demand strict privacy due to regulatory, ethical, or personal concerns, while others may accept looser privacy settings to maximize their utility. This \emph{privacy heterogeneity} introduces a new dimension of complexity that directly impacts the quality and usefulness of aggregated model updates.

To illustrate the impact of privacy heterogeneity, consider a \ac{fl} system for medical data analysis. A client with a sensitive pre-existing condition may opt for a strict privacy setting (a small privacy budget), introducing substantial noise to their updates to minimize information leakage. In contrast, another client participating as a volunteer in a low-risk study may be comfortable with a more lenient setting (a larger privacy budget), thus contributing higher-quality updates. This diversity in privacy preferences is not hypothetical; it is a practical and crucial consideration in domains such as healthcare, finance, and smart devices.

\begin{figure}
	\centering
	\includegraphics[width=1\linewidth]{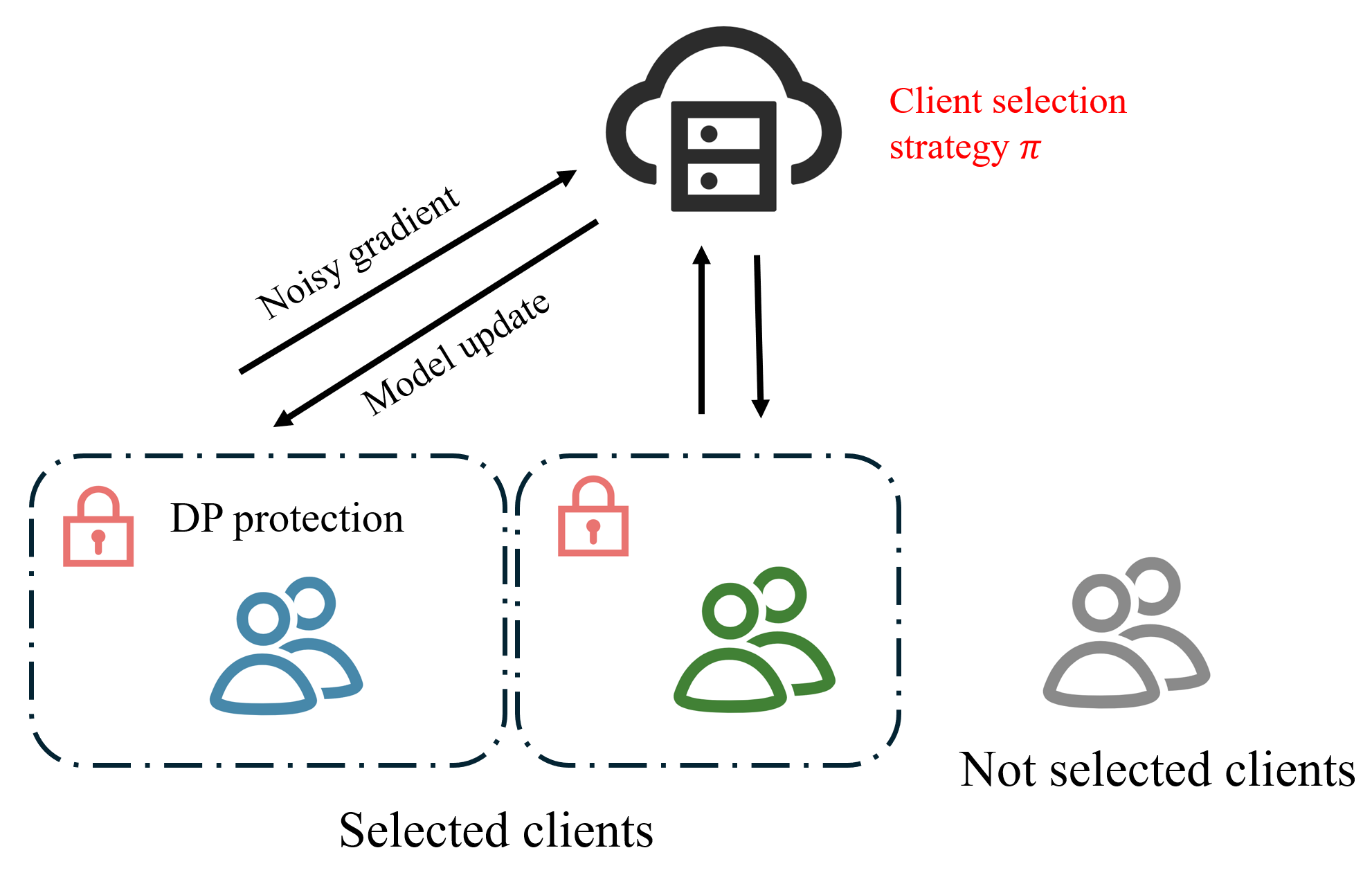}
	\caption{Illustration of differentially private federated learning.}
	\label{fig:diagram}
\end{figure}
The presence of privacy heterogeneity exposes a critical flaw in conventional client selection strategies. 
Most existing approaches select clients based on data quantity or participation frequency, without considering the varying quality of updates caused by different privacy budgets. 
As a result, clients adding high-magnitude noise due to strict privacy constraints are treated equivalently to those providing informative updates. Indiscriminate selection of high-noise clients can significantly impair model convergence and accuracy, and may even amplify the negative effects of privacy-preserving mechanisms such as gradient clipping and noise addition. Therefore, it is essential to design client selection strategies that are \emph{privacy-aware}-that is, strategies that account for both data and privacy heterogeneity to optimize learning performance.

Designing such privacy-aware client selection strategies is nontrivial. 
Client data are often \ac{non-iid}, and the introduction of privacy-preserving noise further complicates the aggregation process. The selection of clients influences both the bias in the aggregated gradients and the overall error induced by privacy mechanisms. Achieving an optimal balance between data utility and privacy protection requires a principled and systematic approach.

These challenges motivate us to investigate the following research questions:

\begin{itemize}
	\item \textbf{How does client selection fundamentally affect the convergence and accuracy of \ac{dpfl} in the presence of privacy heterogeneity?}
	\item \textbf{What is the optimal client selection policy to minimize training loss when both data and privacy heterogeneity are present?}
\end{itemize}

In this paper, we present the first systematic study of privacy-aware client selection in \ac{dpfl}, explicitly addressing the challenges introduced by heterogeneous privacy budgets. \noindent\textbf{Our main contributions are summarized as follows:}
\begin{itemize}
	\item \textbf{Significance:} We identify and rigorously formulate the problem of privacy heterogeneity in \ac{dpfl}, demonstrating that diverse client privacy requirements are a fundamental obstacle to achieving both privacy and accuracy in real-world federated learning systems.
	
	\item \textbf{Challenge-Quantifying the impact of privacy heterogeneity:} We show that existing client selection strategies fail under heterogeneous privacy budgets, as the interplay between privacy levels and data distributions introduces complex noise-utility trade-offs. To address this, we provide the first theoretical convergence analysis of DP-FL under privacy-aware client selection, quantifying how privacy heterogeneity affects model performance.
	
	\item \textbf{Challenge-Designing optimal client selection under heterogeneity:} Selecting clients to maximize model utility is non-trivial when each client contributes updates with different noise levels. We overcome this challenge by proposing a convex optimization-based selection strategy that adaptively balances privacy and utility, ensuring that clients with higher privacy budgets are prioritized without excluding others.
	
	\item \textbf{Key insight:} Our analysis reveals that carefully designed privacy-aware selection can substantially mitigate the negative effects of strict privacy constraints, leading to significant improvements in model accuracy-up to 10\% on CIFAR-10 and 40\% on MNIST/FashionMNIST-compared to conventional selection methods.
\end{itemize}

The rest of the paper is organized as follows.
Section \ref{section: related work} introduces related works on \ac{dpfl} and client selection methods.
Section \ref{sec: system model} introduces the preliminaries of \ac{dp}, \ac{fl}, and the \ac{dp}-\ac{fedavg} algorithm.
Section \ref{sec: convergence} presents our convergence analysis under flexible client selection and discusses the resulting insights. 
Section \ref{sec: selection} describes our proposed privacy-aware client selection strategy.
Section \ref{sec: experiments} evaluates the performance of the proposed client selection strategy.
Section \ref{sec: conclusions} concludes the paper.

\section{Related Work}\label{section: related work}
In this section, we review related work on \ac{dpfl} and client selection approaches in \ac{fl}.  
A comparison is shown in Table \ref{table: comparison}.
\begin{table}
	\centering
	\caption{Comparison of related works.}
	\begin{tabular}{ccc}
		\hline
		\textbf{Ref.} & \textbf{Client Selection Method} & \textbf{DP}\\
		\hline
	\cite{asoodeh2021differentially} & Unbiased selection & \Checkmark\\
		\rowcolor{gray!10}
		\cite{wei2021user} & Unbiased selection  & \Checkmark\\
		\cite{girgis2021shuffled} & Unbiased selection  & \Checkmark\\
		\rowcolor{gray!10}
		\cite{zhang2021understanding} & Unbiased selection  & \Checkmark\\
		\cite{luo2024adaptive} &  Importance sampling based selection & \XSolidBrush\\
		\rowcolor{gray!10}
		UCB-CS\cite{cho2020bandit} & Bandit based selection  & \XSolidBrush\\
		ClusterFL\cite{ouyang2021clusterfl} & Clustering based selection  & \XSolidBrush\\
		\rowcolor{gray!10}
		PyramidFL\cite{li2022pyramidfl} & Utility based selection  & \XSolidBrush\\
		Hybrid-FL \cite{yoshida2020hybrid} &  Data generation based selection  & \XSolidBrush\\
		\rowcolor{gray!10}
		Power-of-choice \cite{cho2022towards} &  Local loss based selection  & \XSolidBrush\\
		\textbf{Our paper} & Privacy-aware selection  &\Checkmark \\
		\hline
	\end{tabular}
	\label{table: comparison}
\end{table}

\subsection{Differentially Private Federated Learning}
Recent studies have analyzed the convergence and privacy-utility trade-offs in DP-FL algorithms from various perspectives, including gradient clipping, shuffling, and information-theoretic bounds~\cite{asoodeh2021differentially,wei2021user,girgis2021shuffled,zhang2021understanding}. These works generally adopt an unbiased client selection strategy, where each client's selection probability is proportional to its local dataset size. This approach ensures unbiased gradient estimation but overlooks the impact of privacy-preserving noise, particularly when clients have heterogeneous privacy budgets. As a result, frequently selecting clients with strict privacy requirements (and thus high noise) can significantly degrade overall model performance. This limitation highlights the need for privacy-aware client selection strategies that explicitly account for heterogeneous privacy preferences.

Another line of research has focused on developing practical DP-FL algorithms to improve learning accuracy under fixed privacy constraints~\cite{zhu2020voting,ijcai2021-217,bu2022automatic}. These works introduce advanced techniques for gradient clipping, aggregation, and knowledge transfer. However, they typically assume uniform privacy budgets across clients, which is unrealistic in many real-world applications. The challenge of heterogeneous privacy budgets, where clients may inject vastly different noise levels, remains largely unaddressed.

\subsection{Client Selection in Federated Learning}
Client selection is a fundamental problem in FL, aiming to address various sources of heterogeneity among clients. Several strategies have been proposed in the literature. For example, UCB-CS~\cite{cho2020bandit} uses a bandit-based approach to select clients without incurring extra communication overhead, while Luo et al.~\cite{luo2024adaptive} employ importance sampling to minimize communication costs. Utility-based methods, such as PyramidFL~\cite{li2022pyramidfl}, define utility functions to enhance both statistical and system efficiency. To tackle data heterogeneity, Hybrid-FL~\cite{yoshida2020hybrid} generates synthetic data for improved client selection, and Power-of-choice~\cite{cho2022towards} selects clients with the highest local loss to accelerate convergence.

Despite their effectiveness in addressing data and system heterogeneity, these methods do not consider \ac{dp} constraints. As such, they are not designed to cope with the additional noise introduced by privacy-preserving mechanisms, nor do they address the complexities arising from clients with diverse privacy budgets.

\section{System Model}\label{sec: system model}
In this section, we present the system model, including the threat model, preliminaries of \ac{fl} and \ac{dp}, and the implementation details of the \ac{dp}-\ac{fedavg} algorithm, depicted in Fig. \ref{fig:diagram}, that forms the basis for our convergence analysis.

\subsection{Threat Model}
We assume the \ac{fl} server operates as an honest-but-curious entity: it faithfully follows the \ac{fl} protocol by aggregating model updates from clients and coordinating the learning process, but may attempt to infer sensitive information from received updates. Such inference could potentially result in privacy breaches, including identity theft or financial harm to individuals.

\subsection{Preliminaries}
\subsubsection{Differential Privacy}
\ac{dp} provides strong guarantees against information leakage by ensuring that the output of a randomized algorithm does not significantly change when a single data point in the input is modified~\cite{abadi2016deep}. This property makes \ac{dp} highly suitable for privacy-preserving machine learning.

\begin{definition}[Differential privacy]
	A randomized algorithm $\mathcal{M}: \mathcal{X} \rightarrow \mathcal{R}$ is said to be ($\epsilon,\delta$)-\ac{dp} if for every pair of neighboring datasets $X, X' \in \mathcal{X}$ that differ in one entry and for any subset of output $\mathcal{S}\subseteq\mathcal{R}$,
	\begin{align}
		\Pr\left[\mathcal{M}(X)\in\mathcal{S}\right]\le e^\epsilon \Pr\left[\mathcal{M}(X')\in\mathcal{S}\right]+\delta.
	\end{align}
\end{definition}

An $(\epsilon, \delta)$-\ac{dp} mechanism guarantees that the privacy loss is bounded by $\epsilon$ with a probability of at least $1-\delta$ \cite{dwork2014algorithmic}.
Here, $(\epsilon, \delta)$ is called the privacy budget; smaller values correspond to stronger privacy guarantees.

A common mechanism to achieve $(\epsilon, \delta)$-\ac{dp} is the Gaussian mechanism, which adds noise to the output proportional to the sensitivity of the function:
\begin{align}\label{equ: Gaussian mechanism}
	\mathcal{M}^\text{Gaussian}(X) = f(X) + \mathcal{N}(0, \sigma^2(\epsilon,\delta,\Delta f)),
\end{align}
where $\sigma^2 = \frac{2(\Delta f)^2\log(\frac{1.25}{\delta})}{\epsilon^2}$ and $\Delta f$ is the sensitivity, the maximum change  in function $f$ by altering a single data point~\cite{dwork2014algorithmic}.

 \subsubsection{Federated Learning}
 \ac{fl} aims to learn a global model from local data held by a set of clients, denoted by $\mathcal{N} = \{1,\cdots,N\}$ \cite{li2020federated2}.
 Each client holds a local dataset $\mathcal{M}_k = \{(x_{k,i},y_{k,i})\}_{i=1}^{N_k}$ consisting of $N_k$ examples.
 The local objective function of client $k$, $f_k(\boldsymbol{w})$, is defined as the average loss over its dataset:
 \begin{align}
 	f_k(\boldsymbol{w}) = \frac{1}{N_k}\sum_{i=1}^{N_k}\mathcal{L}(\boldsymbol{w},x_{k,i},y_{k,i}),
 \end{align}
 where $\mathcal{L}(\boldsymbol{w},x_{k,i},y_{k,i})$ is the loss for model $\boldsymbol{w}\in \mathbb{R}^D$ on the data pair ($x_{k,i}$, $y_{k,i}$) \cite{mcmahan2017communication}.
The global objective is to minimize the weighted sum of local objectives:
 \begin{equation}\label{equ: global objective}
 	\begin{aligned}
 		\min_{\boldsymbol{w}} F(\boldsymbol{w}) = \sum_{k=1}^{N}\frac{|\mathcal{M}_k|}{\sum_{i=1}^{N}|\mathcal{M}_i|}f_k(\boldsymbol{w}),
 	\end{aligned}
 \end{equation}
 which is a weighted sum of $N$ clients' local objective functions, proportional to the sizes of the clients' datasets \cite{mcmahan2017communication}.

\subsection{Federated Learning with Differential Privacy}
We now describe the canonical \ac{dp}-\ac{fedavg} algorithm, which combines \ac{fedavg}~\cite{mcmahan2017communication} with \ac{dp}-SGD~\cite{abadi2016deep}. We also present our general client selection strategy and noise computation process.

\subsubsection{Workflow of DP-FedAvg}
Algorithm~1 outlines the \ac{dp}-\ac{fedavg} procedure. 
At each global iteration $t_g$, the server selects clients according to a pre-defined selection strategy $\pi$ (Line 3), and records how many times each client is selected (Line 4). The server then distributes the current global model to the selected clients.

Each selected client performs $T_l$ local SGD steps on a random subset of its data (Lines 12-17). During each step, the client clips the norm of its computed gradient (Line 15), adds Gaussian noise calibrated to its privacy budget (Line 17), and uploads the noisy update to the server (Line 19). 
The server then aggregates these updates to refine the global model (Line 8). 
The clipping threshold $C$ and the selection count $T_k$ are critical in determining the noise variance for each client.

\begin{algorithm}
	\DontPrintSemicolon
	\label{algo: DP-FedAvg}
	\caption{DP-FedAvg}
	\KwInput{Client selection strategy $\pi$, stepsize $\alpha_{t^g,t^l}$, gradient clipping threshold $C$, clients' datasets $\mathcal{M}_k, \forall k$, privacy budgets $(\epsilon_k, \delta_k), \forall k$, number of global rounds $T^g$, number of local rounds $T^l$, loss function $\mathcal{L}$}
	\KwOutput{Final global model $\boldsymbol{w}_T$}
	
	\textbf{Server executes:}
	
	Initialize $\boldsymbol{w}_{0,0}$
	
	Preselect client sets $\mathcal{S}_1, \cdots, \mathcal{S}_{T^g}$, that participate in the $T^g$ rounds according to strategy $\pi$
	
	For each client $k$, compute the selection count $T_k$.
	
	\For{$\text{each global round}\,\,\, t^g = 0, 1,\cdots, T^g-1$} 
	{
		
		\For{$\text{each client } k \in \mathcal{S}_{t^g}\,\,\, \text{in parallel}$} 
		{compute noisy update $ \tilde{\boldsymbol{g}}_k(\boldsymbol{w}_{t^g,0}) \gets \text{ClientUpdate}(\mathcal{M}_k, \boldsymbol{w}_{t^g,0},\epsilon_k,\delta_k, r_k, T_k, T^l, C)$
		}
		Update the global model: $\boldsymbol{w}_{t^g+1,0} \gets \boldsymbol{w}_{t^g,0} -\alpha_{t^g,0}\oldfrac{1}{|\mathcal{S}_{t^g}|}\sum_{k\in \mathcal{S}_{t^g}}\tilde{\boldsymbol{g}}_k(\boldsymbol{w}_{t^g})$
	}
	\tcc{The ClientUpdate function}
	\textbf{ClientUpdate$(\mathcal{M}_k, \mathcal{L}, \boldsymbol{w}_{t^g,0},\epsilon_k,\delta_k, r_k, T_k, T^l, C)$:}\tcp*{Run on client $k$}
	
	Initialize local model $\bar{\boldsymbol{w}}_{k,0} \gets \boldsymbol{w}_{t^g,0}$
	
	Compute noise variance $ \sigma_k^2 \leftarrow \text{NoiseComputation}(\epsilon_k,\delta_k, r_k, T_k, T^l)$
	
	\For{each local iteration$t^l = 0,1,\cdots,T^l-1$}
	{Uniformly sample a subset of local data $\mathcal{D}_{k,t^l}$ from $\mathcal{M}_k$
		
		\For{each data $d \in \mathcal{D}_{k,t^l}$}
		{
			Compute clipped gradient:
			$\boldsymbol{g}_k(\bar{\boldsymbol{w}}_{k,t^l},d) \!\gets\! \!\frac{\nabla\!_{\bar{\boldsymbol{w}}_{k,t^l}}\!\mathcal{L}(\bar{\boldsymbol{w}}_{k,t^l},\! d)}{\!\max\!\left\{\!1,\!\frac{\left\|\nabla\!_{\bar{\boldsymbol{w}}_{k,t^l}}\!\mathcal{L}(\bar{\boldsymbol{w}}_{k,t^l},d)\right\|_2\!}{C}\!\right\}}$
		}
		Compute average gradient:
		$\bar{\boldsymbol{g}}_k(\bar{\boldsymbol{w}}_{k,t^l}) \gets \oldfrac{1}{|\mathcal{D}_{k,t^l}|} \sum_{d \in \mathcal{D}_{k,t^l}} \boldsymbol{g}_k(\bar{\boldsymbol{w}}_{k,t^l},d)$
		
		Update local model:
		$\bar{\boldsymbol{w}}_{k,t^l+1} \gets \bar{\boldsymbol{w}}_{k,t^l} - \alpha_{t^g,t^l}\left(\bar{\boldsymbol{g}}_k(\bar{\boldsymbol{w}}_{k,t^l}) + \mathcal{N}(0,\sigma_k^2\boldsymbol{I})\right)$
	}
	
	Compute noisy update:
	$\tilde{\boldsymbol{g}}_{k,t} \gets \bar{\boldsymbol{w}}_{k,0}-\bar{\boldsymbol{w}}_{k,T^l}$
	
	return $\tilde{\boldsymbol{g}}_{k,t}$ to the server
	
	\textbf{NoiseComputation($\epsilon_k,\delta_k, r_k, T_k, T^l$):}
	
	Compute noise variance:
	$\sigma_k^2 \gets 8T^lT_kC^2\oldfrac{\log\left(e+\left(\oldfrac{r_k\log\left(1+\oldfrac{1}{r_k}\left(e^{\epsilon_k}-1\right)\right)}{\delta_k}\right)\right)}{|\mathcal{M}_k|^2r_k^2\log\left(1+\oldfrac{1}{r_k}\left(e^{\epsilon_k}-1\right)\right)^2}$
	
	return $\sigma_k^2$
\end{algorithm}

\subsubsection{Client Selection Strategy}\label{subsec: selection}
The client selection strategy, defined in Line~3 of Algorithm~1, plays a pivotal role in DP-FL. 
The server sets the selection probability vector $\boldsymbol{p}^\text{s} = [p^{\text{s}}_{1},\cdots, p^{\text{s}}_{N}]$ before training, where $p^{\text{s}}_k$ is the probability of selecting client $k$. 
In each round, $N^\text{select}$ clients are sampled (with replacement) according to $\boldsymbol{p}^\text{s}$. If a client is sampled multiple times in a round, it acts as multiple virtual clients, each with independently sampled data and independently added noise. The expected objective under this strategy is:
\begin{align}\label{equ: selected objective}
	F^\text{b}(\boldsymbol{w}) = \sum_{k=1}^{N}p^{\text{s}}_{k}f_k(\boldsymbol{w}).
\end{align}
The commonly used unbiased selection probability is $p^{\text{u}}_{k}=\frac{|\mathcal{M}_k|}{\sum_{i=1}^{N}|\mathcal{M}_i|}$, aligning $F^\text{b}(\boldsymbol{w})$, which aligns $F^\text{b}(\boldsymbol{w})$ with the global objective $F(\boldsymbol{w})$. However, when clients have heterogeneous privacy budgets, deviating from $\boldsymbol{p}^\text{u}$ may reduce training error and improve performance.

\subsubsection{Noise Computation}\label{subsubsec: noise compuation}
The noise variance for each client depends on its privacy budget and the number of times it is selected, taking into account privacy composition across iterations. Specifically, the variance is computed to ensure that the cumulative privacy loss over all rounds does not exceed $(\epsilon_k, \delta_k)$. 
Formal privacy composition results are provided in Lemma~\ref{lemma: subsampling} and Lemma~\ref{lemma: strong composition}.

\begin{lemma}[Inverse version of subsampling Lemma \cite{balle2018privacy}]\label{lemma: subsampling}
	Suppose a mechanism $\mathcal{M}$ satisfies $(\log(\frac{(e^\epsilon-1)}{r}+1), \frac{\delta}{r})$-\ac{dp}.
	If data is subsampled by a ratio $r$ without replacement, the mechanism under subsampling satisfies $(\epsilon, \delta)$-\ac{dp}.
\end{lemma} 

\begin{lemma}[Gaussian mechanism with strong composition theorem \cite{kairouz2015composition}]\label{lemma: strong composition}
	For real-valued queries with sensitivity $S>0$, the Gaussian mechanism with noise variance $\frac{8TS^2\log\left(e+\left(\frac{\epsilon}{\delta}\right)\right)}{\epsilon^2}$ satisfies $(\epsilon, \delta)$-\ac{dp} under $T$ iterative usages, for any $\epsilon>0$ and $\delta \in (0,1]$.
\end{lemma}

The strong composition result in Lemma 2 provides a closed-form expression for the required noise variance, which scales as $\Theta(T)$. For clients with varying dataset sizes $\{|\mathcal{M}_k|\}_{k=1}^N$, privacy budgets $\{(\epsilon_k, \delta_k)\}_{k=1}^N$ and subsampling ratio $r_k$, if client $k$ is selected $T_k$ times in the entire training process, the required noise variance per gradient dimension is:
\begin{equation}
	\sigma_k^2 = V_k T_kT_l C^2,
\end{equation}
where
\begin{align}\label{equ: noise level, variance}
	V_k= \oldfrac{8\log\left(e+\left(\oldfrac{r_k\log\left(1+\oldfrac{1}{r_k}\left(e^{\epsilon_k}-1\right)\right)}{\delta_k}\right)\right)}{|\mathcal{M}_k|^2r_k^2\log\left(1+\oldfrac{1}{r_k}\left(e^{\epsilon_k}-1\right)\right)^2},
\end{align}
and $C$ is the clipping threshold. 
This variance, derived from Lemmas 1 and 2, ensures the noise-adding mechanism meets the $(\epsilon_k, \delta_k)$-DP requirements for each client $k\in\mathcal{N}$.

It is important to note that the client sampling probability $p_k^\text{s}$ affects both the aggregation weight and the selection count $T_k$, which in turn impacts $\sigma_k^2$. A smaller $p_k^\text{s}$ typically reduces $T_k$, lowering the noise variance for client $k$. Therefore, the client selection strategy plays a crucial role in the convergence of the learning algorithm.

\section{Convergence Analysis of DP-FedAvg}\label{sec: convergence}
In this section, we establish the theoretical foundation for privacy-aware client selection by analyzing how client selection fundamentally affects the convergence of \ac{dp}-\ac{fedavg}. 
Our analysis is organized as follows: we first introduce key assumptions that enable tractable and relevant analysis; we then derive convergence results for both strongly convex and non-convex objectives; finally, we provide intuitive interpretations of the results, highlighting the interplay between privacy and client selection.

\subsection{Notations}
To facilitate the analysis, we define the following notations with a slight abuse of notation for clarity.
Let $T^g$ and $T^l$ represent the numbers of global rounds and local rounds, respectively, with a total number of \ac{sgd} iterations as $T = T^gT^l$.
For any iteration $t\in [0,\cdots, T]$, we denote the local model of client $k$ at iteration $t$ as $\bar{\boldsymbol{w}}_{k,t}$.
Additionally, we introduce a virtual aggregation sequence, $\tilde{\boldsymbol{w}}_{t} = \sum_{k=1}^{N}p^\text{s}_{k}\bar{\boldsymbol{w}}_{k,t}$, motivated by \cite{li2019convergence}. 
We summarize the notations in Table \ref{table: notation}.

\begin{table}
	\centering
	\caption{Summary of important symbols.}
\begin{tabular}{cc}
	\hline
	\textbf{Symbol} & \textbf{Meaning} \\
	\hline
	$p^\text{s}_k$ & Selection probability of client $k$ \\
	\rowcolor{gray!10}
	$p^\text{u}_k$ & Unbiased selection probability of client $k$ \\
	$\epsilon_k$ & Privacy budget of client $k$ \\
	\rowcolor{gray!10}
	$\sigma_k^2$ & Noise variance of client $k$ \\
	$\mathcal{M}_k$ & Local dataset of client $k$ \\
	\rowcolor{gray!10}
	$r_k$ & Subsampling ratio of client $k$ \\
	$N^\text{select}$ & Number of selected clients per round\\
	\rowcolor{gray!10}
	$f_k$ & Loss function of client $k$\\
	$C$ & Clipping threshold \\
	\rowcolor{gray!10}
	$D$ & Model dimension \\
	$T^g$ & Number of global rounds \\
	\rowcolor{gray!10}
	$T^l$ & Number of local rounds \\
	$T$ & Total number of iterations ($T^gT^l$)\\
	\rowcolor{gray!10}
	$T_k$ & Selected times of client $k$\\ 
	$\boldsymbol{w}_t$ & Global model \\
	\rowcolor{gray!10}
	$\bar{\boldsymbol{w}}_t$ & Local model of client $k$ at iteration $t$ \\
	$\tilde{\boldsymbol{w}}_t$ & Virtual aggregation sequence $\sum_{k=1}^{N}p_k^\textnormal{s}\bar{\boldsymbol{w}}_{k,t}$ \\
	\rowcolor{gray!10}
	$\!\!G^\text{select,clip}\!\!$, $G^{\text{b,clip}}$ &Training loss by client selection and gradient clipping\\
	$G^\text{DP}$ & The training loss due to privacy preservation \\
	\hline
\end{tabular}
\label{table: notation}
\end{table}
\subsection{Assumptions}\label{subsec: assumption}
We introduce the following assumptions for our convergence analysis.
The assumptions are standard in \ac{fl} convergence analyses (\textit{e.g.}, \cite{li2019convergence,gorbunov2021local,stich2018local}).

\begin{assumption}[Lipschitz-continuous objective gradients ($L$-smoothness)]\label{assumption: Lipschitz}
	Each local objective function $f_k$ is continuously differentiable, and its gradient $\nabla f_k$ is Lipschitz continuous with a Lipschitz constant $L>0$.
	Formally, for all $\boldsymbol{w}, \boldsymbol{v} \in \mathbb{R}^D$ and $k\in\mathcal{N}$,
	\begin{align}
		f_k\left(\boldsymbol{w}\right)\le f_k\left(\boldsymbol{v}\right) + \nabla f_k\left(\boldsymbol{v}\right)^\top\left(\boldsymbol{w}-\boldsymbol{v}\right)+\frac{L}{2}\left\|\boldsymbol{w}-\boldsymbol{v}\right\|^2_2.
	\end{align}
\end{assumption}

\begin{assumption}[$\mu$-strong convexity]\label{assumption: convexity}
	Each local objective function $f_k$ is strongly convex, \textit{i.e.}, there exists a constant $\mu>0$ such that for all $\boldsymbol{w}, \boldsymbol{v}\in \mathbb{R}^D$ and $k\in \mathcal{N}$,
	\begin{align}
		f_k\left(\boldsymbol{w}\right)\ge f_k\left(\boldsymbol{v}\right) + \nabla f_k\left(\boldsymbol{v}\right)^\top \left(\boldsymbol{w}-\boldsymbol{v}\right) + \frac{\mu}{2}\left\|\boldsymbol{w}-\boldsymbol{v}\right\|^2_2.
	\end{align}
\end{assumption}

\begin{assumption}[Bounded norm]\label{assumption: norm} 
	The squared norm of gradients is bounded. 
	Specifically, for all $\boldsymbol{w}\in \mathbb{R}^D$ and $k\in\mathcal{N}$,
	\begin{align}
		\left\|\nabla f_k\left(\boldsymbol{w}\right)\right\|_2^2\le B_1^2.
	\end{align}
\end{assumption}

\begin{assumption}[Bounded gradient dissimilarity]\label{assumption: dissimilarity}
	The norm of the difference between the aggregated gradients at local steps and virtual steps is bounded.
	Formally, for all $t \le T$,
	\begin{align}
		\left\|\frac{1}{N}\sum_{k=1}^{N}\nabla f_k(\tilde{\boldsymbol{w}}_t)-\frac{1}{N}\sum_{k=1}^{N}\nabla f_k(\bar{\boldsymbol{w}}_{k,t})\right\|_2 \le B_2.
	\end{align}
\end{assumption}

\subsection{Side Effect of Gradient Clipping}\label{appendix: clipping}
Although gradient clipping (Line 15 in Algorithm~1) effectively limits the sensitivity (as shown in Lemma~\ref{lemma: strong composition}) and reduces the required noise variance~\cite{abadi2016deep}, it may introduce additional training error. 
In \ac{dpfl}, clients clip their gradients locally without access to the global gradient information, which can alter the direction of the aggregated gradient. 
Specifically, the expected virtual aggregation of clipped gradients under unbiased client selection is given by:
\begin{equation}\label{equ: g_clip}
	\begin{aligned}
		&\boldsymbol{g}^\text{clip,unbiased}\left(\tilde{\boldsymbol{w}}_t\right) = \sum_{k=1}^{N}\frac{|\mathcal{M}_k|}{\sum_{i=1}^{N}|\mathcal{M}_i|}\\
		&\mathbb{E}_{\mathcal{D}_{k,t}}\!\left[\frac{1}{|\mathcal{D}_{k,t}|}\!\sum_{d \in \mathcal{D}_{k,t}}\!\frac{1}{\max\{1,\frac{\left\| \boldsymbol{g}_k(\tilde{\boldsymbol{w}}_{t},d)\right\|_2}{C}\}}\boldsymbol{g}_k(\tilde{\boldsymbol{w}}_{t},d)\right]\!.
	\end{aligned}
\end{equation}
where $\boldsymbol{g}_k(\tilde{\boldsymbol{w}}_t, d)$ denotes the local gradient of client $k$ at virtual model $\tilde{\boldsymbol{w}}_t$ for data sample $d$.

In contrast, the clipped aggregated gradient under general (possibly biased) client selection is given as:
\begin{equation}\label{equ: nabla F_clip}
	\begin{aligned}
		&\nabla F^\text{clip}(\tilde{\boldsymbol{w}}_t) 
		= \sum_{k=1}^{N} \frac{|\mathcal{M}_k|}{\sum_{i=1}^{N}|\mathcal{M}_i|}\\ &\mathbb{E}_{\mathcal{D}_{k,t}}\!\left[\frac{1}{|\mathcal{D}_{k,t}|}\sum_{d \in \mathcal{D}_{k,t}}\frac{q_{k,t}(d)}{\max\{1,\frac{\left\| \boldsymbol{g}_k(\tilde{\boldsymbol{w}}_t,d)\right\|_2}{C}\}}\boldsymbol{g}_k(\tilde{\boldsymbol{w}}_t,d)\right],
	\end{aligned}
\end{equation}
where the weighting factor $q_{k,t}(d)$ is defined as:
\begin{align}\label{equ: p_c}
	q_{k,t}(d)=\frac{\max\{1,\frac{\left\| \boldsymbol{g}_k(\tilde{\boldsymbol{w}}_t, d)\right\|_2}{C}\}}{\max\{1,\frac{\left\|\nabla F(\tilde{\boldsymbol{w}}_t)\right\|_2}{C}\}}, \forall k\in \mathcal{N}.
\end{align}

The discrepancy between $\boldsymbol{g}^\text{clip,unbiased}\left(\tilde{\boldsymbol{w}}_t\right)$ and $\nabla F^\text{clip}(\tilde{\boldsymbol{w}}_t)$ introduces a clipping error, which can impact the convergence of the training process.

\subsection{Main Result}
The main results of our convergence analysis are presented in Theorems 1 and 2. To facilitate the discussion, we first define  $G^\text{select, clip}$, which captures the training loss resulting from client selection and gradient clipping:
\begin{align}
	&G^\text{select,clip} = \underbrace{\left\| \boldsymbol{p}^\text{s}\!-\!\boldsymbol{p}^\text{u}\right\|_1}_\text{selection gap}+\underbrace{\max_t\{\left\|\Delta_{t}\right\|_1\}}_\text{clipping gap} + \frac{B_2}{C},
\end{align}
where $\boldsymbol{p}^\text{u}$ is the unbiased selection probability vector with entries $\frac{|\mathcal{M}_k|}{\sum_{i=1}^{N}|\mathcal{M}_i|}$, and
$\Delta_{t}$ is a vector whose $k^{th}$ entry is $\oldfrac{|\mathcal{M}_k|}{\sum_{i=1}^{N}|\mathcal{M}_i|}\oldfrac{\sum_{d \in \mathcal{M}_{k}}|q_{k,t}(d)-1|}{|\mathcal{M}_{k}|}$.

Similarly, the training loss due to privacy preservation is given by
\begin{align}
	G^\text{DP}= \sum_{k=1}^{N}(p^\text{s}_{k})^2DV_k,
\end{align}
where $V_k$ is defined in Equation (\ref{equ: noise level, variance}). 
Additionally, the clipping gap is characterized by
\begin{align}
	&G^{\text{b,clip}} =\max_t\{ \left\|\Delta_{t}^\text{b}\right\|_1\} +\frac{B_2}{C},
\end{align}	
which captures the loss introduced by gradient clipping, with $\Delta_{t}^\text{b}$ being a vector whose $k^{th}$ entry is $\left(\frac{\sum_{d \in \mathcal{M}_{k}}|q_{k,t}(d)-1|}{|\mathcal{M}_{k}|}\right)p^{\text{s}}_{k}$.

\subsubsection{Strongly convex objective functions}
The following theorem establishes the convergence result for DP-FedAvg with strongly convex objective functions.
\begin{theorem}\label{theorem: convergence}
	Suppose that Assumptions 1-3 hold.
	Choose a stepsize $\alpha_t = \oldfrac{\beta}{\gamma+t}$ with $\beta > \oldfrac{B_1}{C\mu}$ and $\gamma\ge 1$.
	Then, for any $\Gamma>0$, the expected training loss of \ac{dp}-\ac{fedavg}  satisfies:
	\begin{equation}\label{equ: convergence}
		\begin{aligned}
			&\sqrt{\mathbb{E}\left[F(\boldsymbol{w}_{T+1})-F^*\right]} \le\\
			&\!\underbrace{\frac{Z}{\sqrt{\gamma+T}}}_\textnormal{vanishing error}+\sqrt{18\Gamma\frac{\beta TL}{\gamma\!+\!T}\frac{B_1C}{\mu}G^\textnormal{b,clip}}+\frac{\sqrt{2L}\beta TC}{2(\gamma+T)} \\
			&\left[\!G^\textnormal{select,clip}\!+\!\sqrt{(1\!+\!\frac{1}{\Gamma})(G^\textnormal{select,clip})^2\!+\!(1\!+\!9\Gamma)G^\textnormal{DP}}\right]\!,
		\end{aligned}
	\end{equation}
	where 
	\begin{equation}
		\begin{aligned}
			&Z =\\ &\sqrt{\frac{9}{2}L\Gamma\gamma\left\|\boldsymbol{w}_{1}-\boldsymbol{w}^{\textnormal{b},*}\right\|^2_2\!+\!\frac{1+9\Gamma}{2}\frac{\beta^2TL}{\gamma+T}C^2\!+\!\gamma\left(F(\boldsymbol{w}_1)\!-\!F^*\right)},
		\end{aligned}
	\end{equation}
	$F^*$ is the optimal value of the global objective $F$ in Equation (\ref{equ: global objective}), and
	$\boldsymbol{w}^{\text{b},*}$ is the optimal solution of the biased objective $F^\text{b}$ in Equation (\ref{equ: selected objective}).
\end{theorem}

We defer the proof of Theorem \ref{theorem: convergence} to Appendix A.

Theorem \ref{theorem: convergence} demonstrates that the square root of the expected training error is upper bounded by two components: a vanishing error and a non-vanishing error. 
The vanishing error, which is influenced by initialization and stochastic gradient variance, decreases at a rate of $\mathcal{O}(\frac{1}{\sqrt{T}})$ and diminishes as $T \rightarrow \infty$.
The non-vanishing error increases with the total number of iterations $T$ but converges to a constant as $T\rightarrow \infty$.

This theorem reveals a fundamental trade-off in DP-FL
systems. 
The convergence error decomposes into three interpretable components: clipping
error (which is unavoidable when bounding gradient sensitivity), the selection and privacy error (the central challenge addressed in this work), and the vanishing error (which diminishes with more iterations). 
Particularly, the
privacy error, $G_\text{DP} = \sum_{k=1}^N (p_k^\text{s})^2 DV_k$, grows quadratically with the client selection probabilities. 
This indicates that frequently selecting high-noise clients can significantly degrade performance.

\begin{remark}
	The assumption of strong convexity (Assumption \ref{assumption: convexity}) enables the analysis of convergence under a stepsize of $\Theta\left(\frac{1}{t}\right)$, which effectively mitigates the impact of noise in later iterations.
\end{remark}
\subsubsection{Non-convex objective functions}
We now characterize the convergence behavior of \ac{dp}-\ac{fedavg} on non-convex objective functions.
\begin{theorem}\label{theorem: convergence_smooth}
	Suppose that Assumptions 1 and 3 hold.
	Choose a stepsize $\alpha_t = \oldfrac{B_1}{C}\oldfrac{1}{\sqrt{T}}$.
	Then, \ac{dp}-\ac{fedavg} satisfies:
	\begin{equation}\label{equ: convergence_smooth}
		\begin{aligned}
			&\sqrt{\frac{1}{T}\sum_{t=1}^{T}\mathbb{E}\left[\left\|\nabla F(\tilde{\boldsymbol{w}}_t)\right\|_2^2\right]}\le\\
			&\underbrace{\frac{B_1}{2}\!\!\!\left[\!G^\textnormal{select,clip}\!+\!\!\sqrt{\!\left(\!G^\textnormal{select,clip}\right)^2 \!\!+\! 2\!\sum_{k=1}^{N}(p^\text{s}_{k})^2D\sqrt{T}LV_k}\!\right]}_\textnormal{non-vanishing error}\\
			&+\underbrace{\sqrt{\frac{F(\boldsymbol{w}_1)\!-\!F(\boldsymbol{w}^*)}{\sqrt{T}}\!+\!\frac{1}{2} \oldfrac{1}{\sqrt{T}}LB_1^2}}_\textnormal{vanishing error}.
		\end{aligned}
	\end{equation}
\end{theorem}
We defer the proof of Theorem \ref{theorem: convergence_smooth} to Appendix B.
Similar to Theorem \ref{theorem: convergence}, the upper bound on the square root of the expected training error consists of a non-vanishing error and a vanishing error.
However, unlike the error bound for strongly convex objectives in Theorem \ref{theorem: convergence}, the error due to noise insertion increases with the number of iterations as $\mathcal{O}(T^{1/4})$, and the vanishing error decreases as $\mathcal{O}(T^{-1/4})$.

\subsection{Insights and Corollaries}\label{subsec: insights}
The following corollaries, derived from Theorems \ref{theorem: convergence} and \ref{theorem: convergence_smooth}, show that \ac{dp}-\ac{fedavg} converges in expectation to an irreducible training error which is caused by client selection, gradient clipping, and the noise from \ac{dp} protection.
\begin{corollary}\label{corollary: infinite}
	When Assumptions 1-3 hold, the asymptotic square root of training error is upper bounded as follows:
	\begin{equation}\label{inequality: final gap inf}
		\begin{aligned}
			&\lim_{T \rightarrow \infty}\sqrt{\mathbb{E}\{F(\boldsymbol{w}_{T})\! -\! F^*\}} \!\le\!\sqrt{18\Gamma \beta L\frac{BC}{\mu}G^\textnormal{b,clip}}+\\
			&\frac{\sqrt{L}\beta C}{\sqrt{2}}\!\!\left[G^\textnormal{select,clip} \!\!+\!\!\sqrt{(1\!+\!\frac{1}{\Gamma})(G^\textnormal{select,clip})^2\!\!+\!(1\!+\!9\Gamma)G^\textnormal{DP}}\right].\!\!\!
		\end{aligned}
	\end{equation}
\end{corollary}

\begin{remark}
	In classical optimization with Assumptions 1-3 holding (without \ac{dp}), the training error upper bound is $\mathcal{O}(\frac{1}{\sqrt{T}})$, implying convergence to $0$ as $T\rightarrow \infty$ \cite{bottou2018optimization}.
	However, as shown in Corollary \ref{corollary: infinite}, the optimality gap converges to a non-zero error due to \ac{dp} considerations.
\end{remark}

\begin{corollary}\label{corollary: optimal T}
	When Assumptions 1 and 3 hold, the square root of the error upper bound does not monotonically decrease with $T$.
\end{corollary}
\begin{proof}
	The upper bound in Theorem \ref{theorem: convergence_smooth} can be expressed as:
	\begin{equation}
	\begin{aligned}
		&\frac{B_1}{2}\!\!\!\left[\!G^\textnormal{select,clip}\!+\!\!\sqrt{\!\left(\!G^\textnormal{select,clip}\right)^2 \!\!+\! 2\!\sum_{k=1}^{N}(p^\text{s}_{k})^2D\sqrt{T}LV_k}\!\right]\\
		&= \frac{B_1}{2}G^\text{select,clip} + f(T), 
	\end{aligned}
\end{equation}
	where 
	\begin{equation}
		\begin{aligned}
			&f(T) = \frac{B_1}{2} \sqrt{A+Q\sqrt{T}} + \sqrt{R\frac{1}{\sqrt{T}}},
		\end{aligned}
	\end{equation}
	with 
	\begin{equation}
		\begin{aligned}
			&A = (G^\text{select,clip})^2,
			Q = 2\sum_{k=1}^{N}(p_k^\text{s})^2DLV_k,\\
			&R = F(\boldsymbol{w}_1) - F(\boldsymbol{w}^*) + \frac{1}{2}LB_1^2.
		\end{aligned}
	\end{equation}
	The derivative of $f(T)$ is
	\begin{align}\label{equ: f(T)}
		\frac{df(T)}{dT} = \underbrace{\frac{B_1}{8}\frac{Q}{\sqrt{A+Q\sqrt{T}}}\frac{1}{\sqrt{T}}}_{\Theta(T^{-\frac{3}{4}})} - \underbrace{\frac{\sqrt{R}}{4}\frac{1}{T^{\frac{5}{4}}}}_{\Theta(T^{-\frac{5}{4}})}. 
	\end{align}
	Since the first term is $\Theta(T^{-\frac{3}{4}})$ while the second term is $\Theta(T^{-\frac{5}{4}})$, the first term dominates for sufficiently large $T$, making $\frac{df(T)}{dT} > 0$.
	Thus, if the derivative $f'(2)=\oldfrac{df(T)}{dT}|_{T=2} \ge 0$, the error bound increases monotonically; if $f'(2) < 0$, the error bound initially decreases, but eventually increases as $T$ grows.
\end{proof}

\begin{remark}
	Suppose Assumptions 1 and 3 hold. 
	Classical optimization (without \ac{dp}) yields a training error upper bound the decreases as $\mathcal{O}(T^{-1/4})$.
	In contrast, \ac{dp} introduces an additional non-vanishing error term that grows as $\mathcal{O}(T^{1/4})$, resulting in non-monotonic error behavior. 
\end{remark}

Squaring both sides of the error bounds in Theorems \ref{theorem: convergence} and \ref{theorem: convergence_smooth} reveals cross-product terms involving client selection, gradient clipping, and \ac{dp} errors.
This suggests that privacy protection can amplify the training loss due to client selection and clipping, offering a theoretical explanation for empirical observations in the recent literature \cite{bagdasaryan2019differential}. 

\section{Privacy-Aware Client Selection Strategy}\label{sec: selection}

Based on the convergence analysis of DP-FedAvg in Section \ref{sec: convergence}, we now propose a privacy-aware client selection strategy to optimize performance in \ac{dpfl}. 
The strategy operates as shown in Fig. \ref{fig:workflow}.
\begin{figure}
	\centering
	\includegraphics[width=0.7\linewidth]{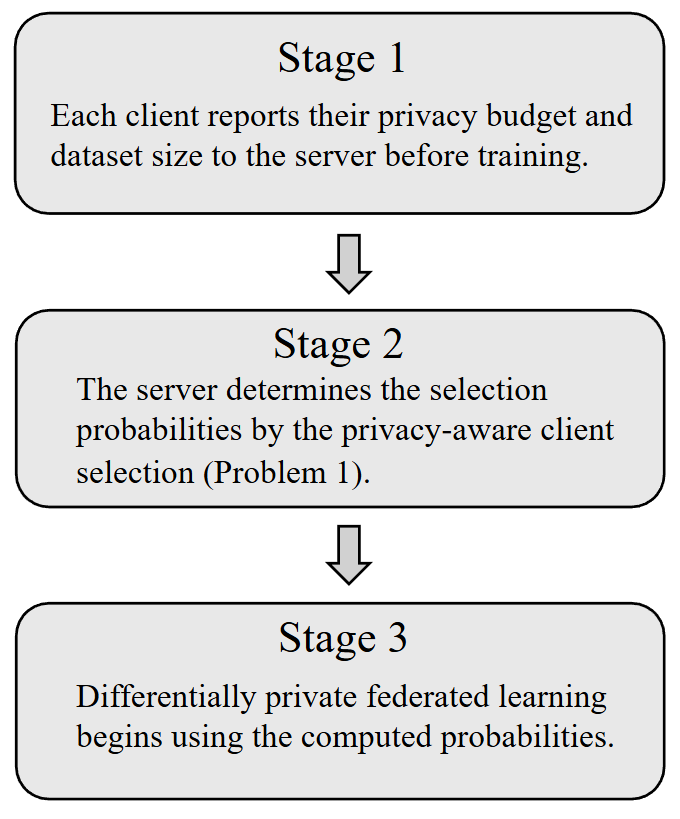}
	\caption{Workflow of the privacy-aware client selection.}
	\label{fig:workflow}
\end{figure}
       
\subsection{Derivation of the Privacy-Aware Client Selection Problem}
We formulate the privacy-aware client selection optimization problem based on the convergence analysis presented in Section~\ref{sec: convergence}. 
This analysis shows that three main factors-client selection, gradient clipping, and DP noise-fundamentally influence the training error in \ac{dp}-\ac{fedavg}. 
Given a fixed gradient clipping threshold $C$, we aim to optimize client selection to minimize the persistent error term described in equation (\ref{equ: convergence_smooth}).

To simplify the optimization, we ignore the clipping gap and define the selection error as $\ell_1$-norm between the selection probability vectors, \textit{i.e.}, $G^\text{select, clip}=\left\|\boldsymbol{p}^\text{s}-\boldsymbol{p}^\text{u}\right\|_1$.
Since we do not know the smoothness constant $L$ of the objective function in advance, we introduce a tunable hyperparameter $\eta = D\sqrt{TL}$ and adjust it during training.

In summary, we solve the following optimization problem:
\begin{tcolorbox}[left = 0.5pt,top=0.5pt,bottom=0.5pt]
	\begin{problem}{Privacy-aware client selection.}\label{opt: original}
		\begin{equation}\nonumber
			\begin{aligned}
				\min_{\boldsymbol{p}^\textnormal{s}}&\left\| \boldsymbol{p}^\textnormal{s}-\boldsymbol{p}^\textnormal{u} \right\|_1\!+\sqrt{\!\left\| \boldsymbol{p}^\textnormal{s}-\boldsymbol{p}^\textnormal{u} \right\|_1^2+\eta \sum_{k=1}^{N}(p^{\textnormal{s}}_{k})^2DV_k}\\
				\textnormal{s.t.}& \text{ }\sum_{k=1}^{N}p^\textnormal{s}_{k} = 1,\text{ } p^\textnormal{s}_{k} \ge 0, \forall k \in \mathcal{N}.
			\end{aligned}
		\end{equation}
	\end{problem}
\end{tcolorbox}

\subsection{Characterization of Problem \ref{opt: original}}
\begin{proposition}\label{proposition: convex}
	Problem \ref{opt: original} is a convex optimization problem.
\end{proposition}
The proof of Proposition 1 is deferred to the Appendix C.

Because the objective is convex and the constraints are linear, we can efficiently find the global optimum of Problem 1 using standard convex optimization solvers.

Next, we characterize the optimal solution of Problem \ref{opt: original}.
\begin{proposition}[Full-participation]\label{proposition: full participation}
	At the optimal solution of Problem \ref{opt: original}, the sever assigns positive selection probabilities to all clients, i.e., $p_{k}^\text{s}>0$, for all $k\in\mathcal{N}$.
\end{proposition}
We defer the proof of Proposition \ref{proposition: full participation} to the Appendix D.
\begin{figure*}[t]
	\centering
	\includegraphics[width=1.02\linewidth]{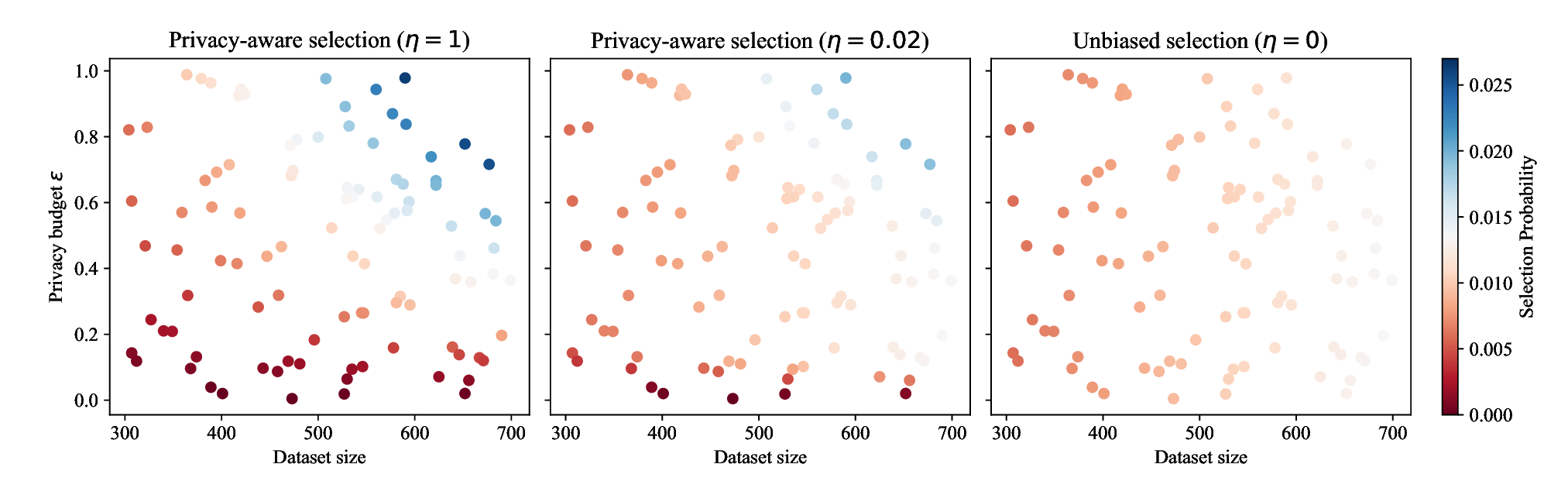}
	\caption{Selection probability of 100 clients with heterogeneous dataset sizes and privacy budgets ($\epsilon$).
		Subsampling ratio $r$ is fixed to 0.1 for all clients.}
	\label{fig:opt_color}
\end{figure*}

\begin{remark}
	Proposition \ref{proposition: full participation} implies that, to minimize training loss, the server should include all clients in the selection process, regardless of their privacy budgets. 
	However, in practice, other constraints such as communication costs or system limitations may require selecting only a subset of clients.
\end{remark}

The proposed privacy-aware client selection strategy also applies to scenarios with \ac{non-iid} public datasets. 
In these cases, we can model public datasets as clients with infinite privacy budgets, allowing them to be seamlessly integrated into the optimization framework without any additional modifications.

\section{Numerical Results}\label{sec: experiments}
In this section, we evaluate our proposed privacy-aware client selection strategy in DP-FL under heterogeneous privacy budgets. 
We compare its performance with existing baseline methods on MNIST, Fashion-MNIST, and CIFAR-10 datasets, and analyze the effects of key parameters such as subsampling ratio and privacy budgets.
It is worth noting that single-machine simulations of \ac{dpfl} face two critical performances bottlenecks that limit their scale:
	\begin{itemize}
		\item A memory bottleneck, where serially loading data for numerous clients causes severe cache thrashing, making the process memory- rather than compute-bound.
		\item A compute bottleneck, stemming from the large overhead of the per-example gradient clipping needed to satisfy \ac{dp} requirements.
	\end{itemize}
	Consequently, numerical experiments are often restricted to smaller-scale benchmarks such as MNIST and CIFAR-10 \cite{zhang2025locally,zhu2025randomized,wang2025codp}.

\subsection{Setup}\label{subsec: setup}
To simulate realistic federated learning scenarios, we introduce heterogeneity both in client dataset sizes and data distributions. 

\textbf{Client Dataset Sizes:} The size of each client's local dataset $|\mathcal{M}_k|$ is determined by $|\mathcal{M}_k| = |\mathcal{S}| v_k / \sum_{i=1}^N v_i$, where each $v_k$ is sampled independently from $\mathcal{U}(0.5, 1.5)$\footnote{Here $\mathcal{U}(x, y)$ refers to a uniform distribution whose interval is $(x, y)$.}\footnote{The argument $(x,y)$ in $\mathcal{U}(x,y)$ can be any value. 
We choose $(x,y)=(0.5, 1.5)$ here for simplicity.}. 
This ensures that clients have varying amounts of data, reflecting practical FL deployments.

\textbf{Data Distribution (Non-IID):} We control the level of data heterogeneity using a mixing parameter $s$. For each client, $s\%$ of their local data is drawn IID from the global training set, while the remaining $(100-s)\%$ is assigned in a non-IID manner by sorting data according to class labels, following~\cite{karimireddy2020scaffold}. This partitioning is applied to MNIST~\cite{lecun1998gradient}, FashionMNIST~\cite{xiao2017fashion}, and CIFAR-10~\cite{krizhevsky2009learning}.

\subsection{Baselines}\label{subsec: baselines}
We compare our privacy-aware client selection strategy with two widely studied categories of baseline mechanisms:

\begin{itemize}
	\item \textbf{Unbiased client selection:} These methods select clients with probabilities proportional to their local dataset sizes, i.e., $\boldsymbol{p}^\text{s} = \boldsymbol{p}^\text{u}$~\cite{zhang2022understanding,bietti2022personalization}. Unbiased client selection methods also include importance compensation approaches, which weight each client's update using a compensation parameter to maintain an unbiased learning objective~\cite{luo2024adaptive}.
	\item \textbf{Biased client selection:} These methods select clients based on criteria unrelated to dataset size, so selection probabilities do not match data volume. Since most biased selection methods ignore privacy considerations, we adopt a widely used approach that selects clients with the largest local training losses~\cite{cho2022towards}.
\end{itemize}

\subsection{Optimal Solution and the Learning Performance}\label{subsec: training performance}
\begin{figure*}[t]
	\begin{subfigure}{0.49\linewidth}
		\centering
		\includegraphics[width=0.9\linewidth]{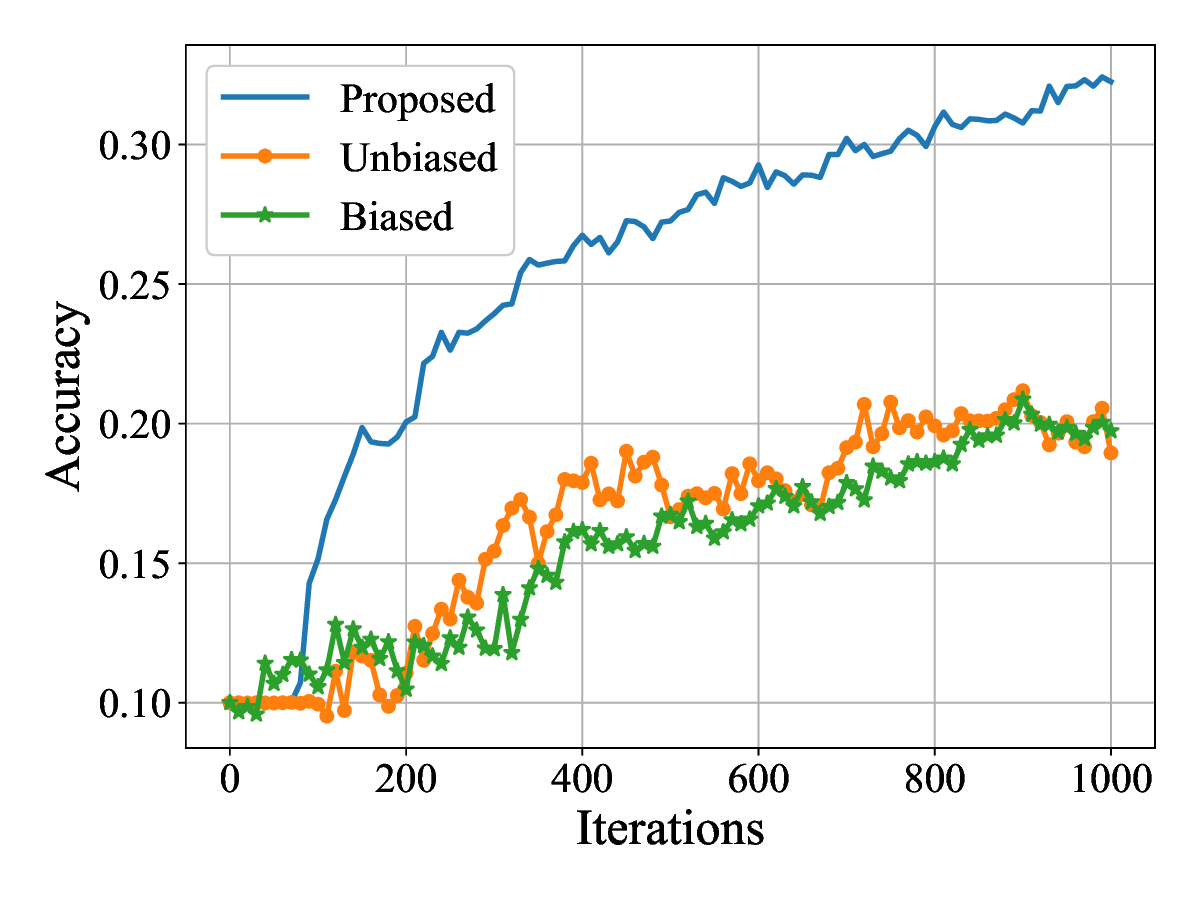}
		\caption{CIFAR-10, Similarity: $s=100$}
		\captionsetup{justification=centering,margin=2cm}
		\label{fig: test_loss_c_iid}
	\end{subfigure}
	\begin{subfigure}{0.49\linewidth}
		\centering
		\includegraphics[width=0.9\linewidth]{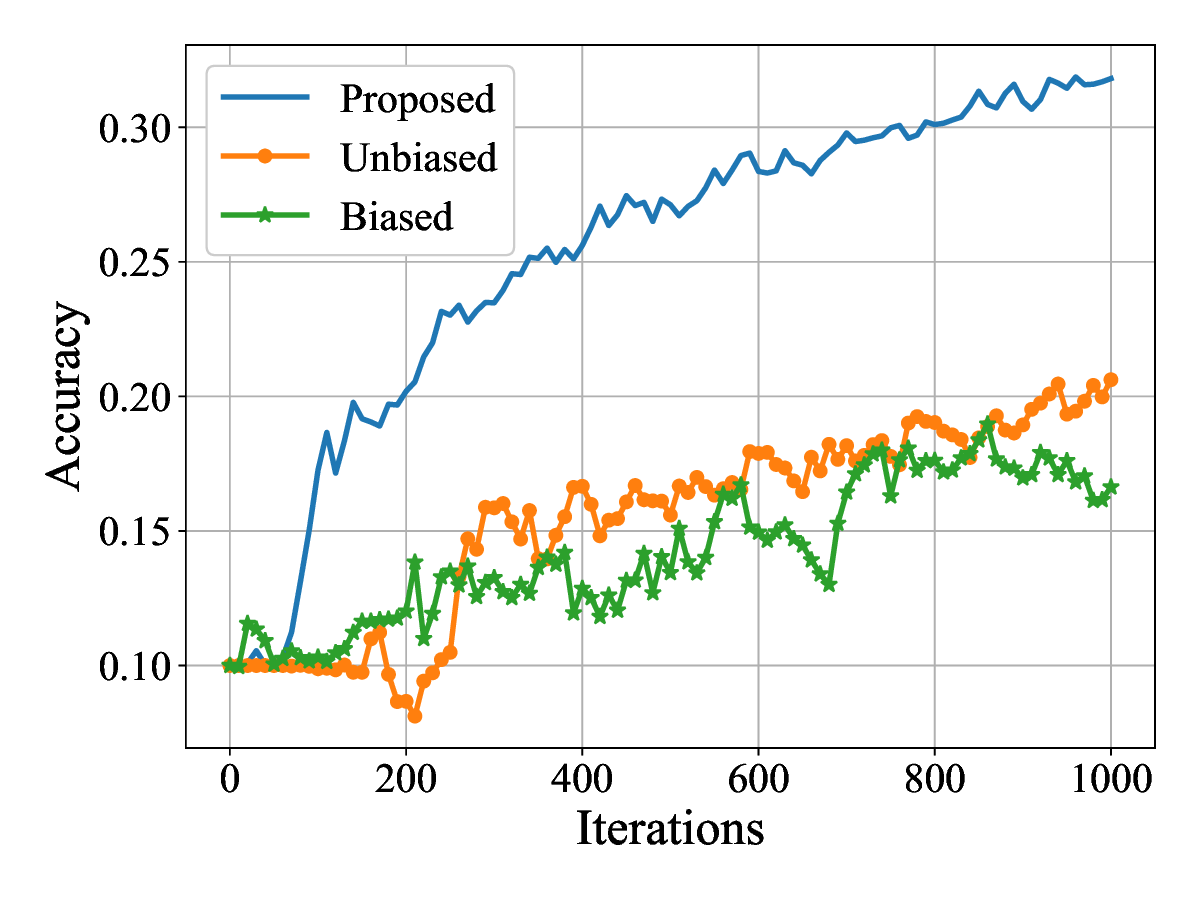}
		\caption{CIFAR-10, Similarity: $s=70$}
		\captionsetup{justification=centering}
		\label{fig: test_loss_c_sim7}
	\end{subfigure}
	\begin{subfigure}{0.49\linewidth}
		\centering
		\includegraphics[width=0.9\linewidth]{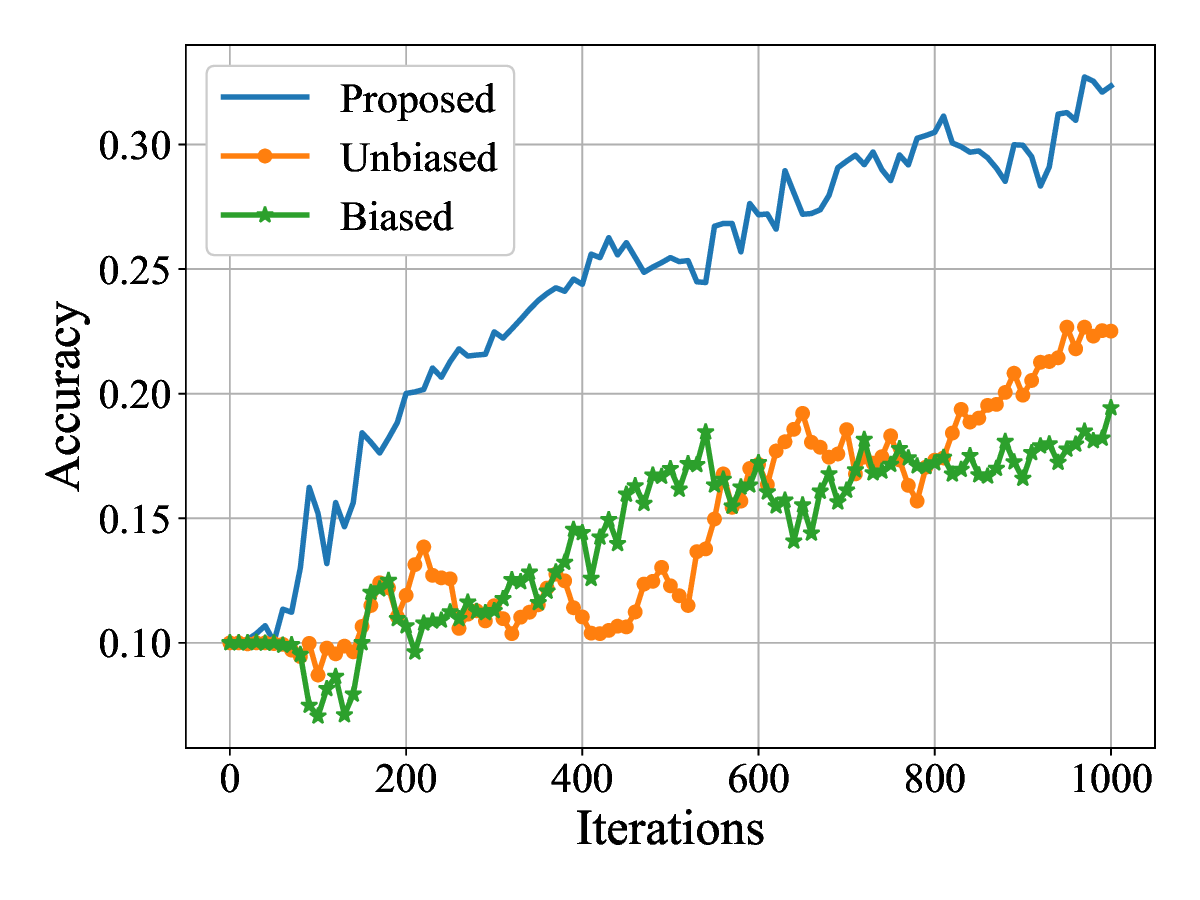}
		\caption{CIFAR-10, Similarity: $s=30$}
		\captionsetup{justification=centering}
		\label{fig: test_loss_c_sim3}
	\end{subfigure}
	\begin{subfigure}{0.49\linewidth}
		\centering
		\includegraphics[width=0.9\linewidth]{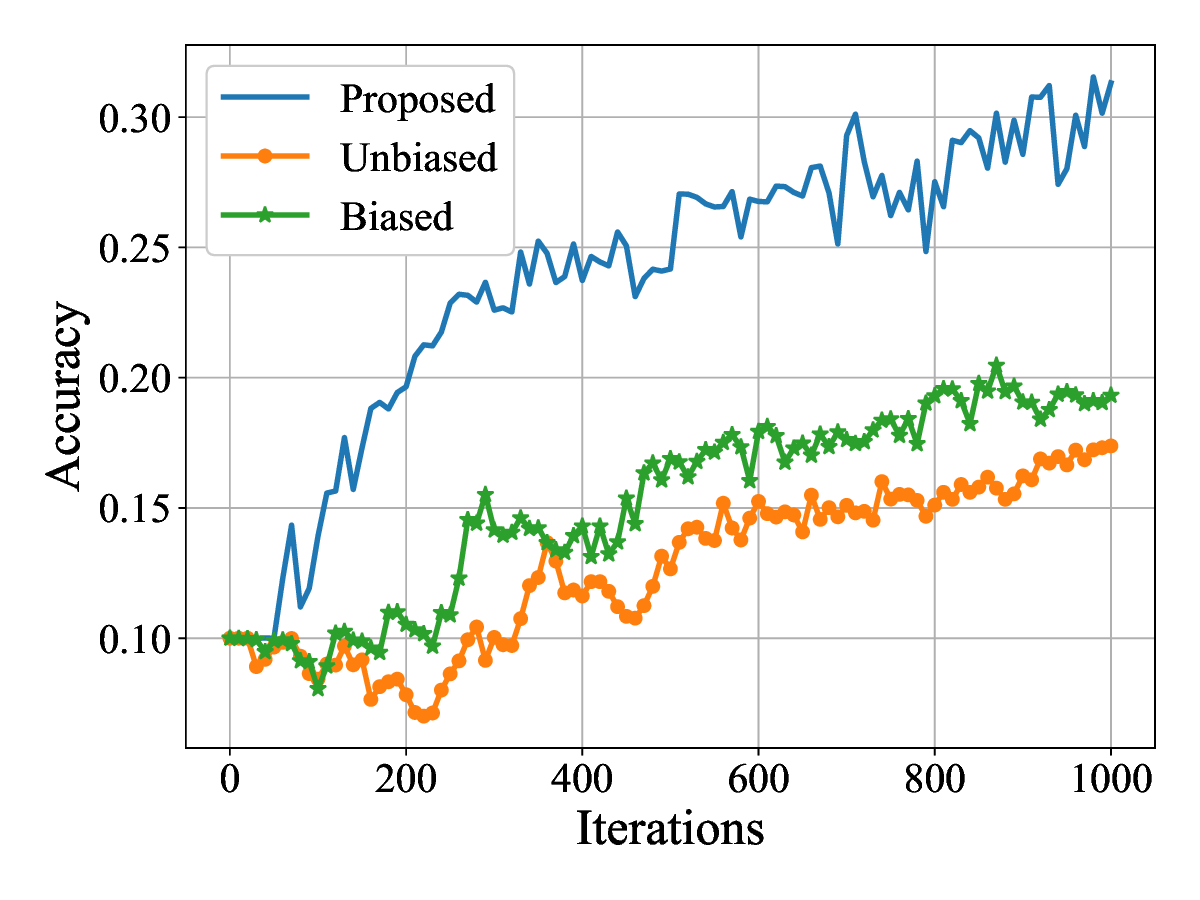}
		\caption{CIFAR-10, Similarity: $s=0$}
		\captionsetup{justification=centering}
		\label{fig: test_loss_c_sim0}
	\end{subfigure}
	\caption{Test accuracy on CIFAR-10.}
	\label{fig: test_acc}
\end{figure*}

\begin{table}[ht]
	\centering
	\normalsize 
	\caption{Test accuracy performance (MNIST).}
	\label{table: test acc 1}
	\begin{tabular}{ccc}
		\hline
		\textbf{Mechanism} & \textbf{Non-IID Degree} & \textbf{Test Accuracy} \\
		\hline
		\multirow{4}{*}{Privacy-aware}  & $s=100$ (IID) & \textbf{68.71\%}\\
		\cline{2-3}
		&  $s=70$ & \textbf{65.12\%}\\
		\cline{2-3}
		&  $s=30$ & \textbf{68.50\%} \\
		\cline{2-3}
		&  $s=0$ & \textbf{67.63\%} \\
		\hline
		\multirow{4}{*}{Unbiased} & $s=100$ (IID)& 28.02\% \\
		\cline{2-3}
		&  $s=70$ & 14.69\% \\
		\cline{2-3}
		&  $s=30$ & 17.94\% \\
		\cline{2-3}
		&  $s=0$ & 14.92\% \\
		\hline
		\multirow{4}{*}{Biased} & $s=100$ (IID)& 32.73\% \\
		\cline{2-3}
		&  $s=70$ & 35.42\% \\
		\cline{2-3}
		&  $s=30$ & 16.98\% \\
		\cline{2-3}
		&  $s=0$ & 18.55\% \\
		\hline
	\end{tabular}
\end{table}
\begin{table}[ht]
	\centering
	\normalsize 
	\caption{Test accuracy performance (FashionMNIST).}
	\label{table: test acc 2}
	\begin{tabular}{ccc}
		\hline
		\textbf{Mechanism} & \textbf{Non-IID Degree} & \textbf{Test Accuracy} \\
		\hline
		\multirow{4}{*}{Privacy-aware}  & $s=100$ (IID) & \textbf{53.53\%}\\
		\cline{2-3}
		&  $s=70$ & \textbf{55.26\%}\\
		\cline{2-3}
		&  $s=30$ & \textbf{54.80\%} \\
		\cline{2-3}
		&  $s=0$ & \textbf{48.94\%} \\
		\hline
		\multirow{4}{*}{Unbiased} & $s=100$ (IID)& 14.68\% \\
		\cline{2-3}
		&  $s=70$ & 16.47\% \\
		\cline{2-3}
		&  $s=30$ & 16.27\% \\
		\cline{2-3}
		&  $s=0$ & 10.69\% \\
		\hline
		\multirow{4}{*}{Biased} & $s=100$ (IID)& 23.63\% \\
		\cline{2-3}
		&  $s=70$ & 20.95\% \\
		\cline{2-3}
		&  $s=30$ & 15.39\% \\
		\cline{2-3}
		&  $s=0$ & 20.37\% \\
		\hline
	\end{tabular}
\end{table}

We begin by visualizing how local dataset sizes and privacy budgets influence the optimal client selection. By solving the optimization problem (Problem~\ref{opt: original}) using CVXPY\footnote{The license of CVXPY: https://www.cvxpy.org/license/index.html.} ~\cite{diamond2016cvxpy}, we obtain the optimal selection probabilities for 100 clients, as shown in Fig.~\ref{fig:opt_color}. In this figure, each client is represented by a point, with its dataset size on the horizontal axis, privacy budget $\epsilon_k$ on the vertical axis, and selection probability indicated by color. Notably, clients with larger datasets and higher privacy budgets are more likely to be selected by the server---this is reflected by a shift in color from red to blue.

The parameter $\eta$ in Problem~\ref{opt: original}, which controls the weight on \ac{dp} error, directly affects the deviation of the client selection probabilities from the unbiased baseline. This effect is illustrated in Fig.~\ref{fig:opt_color}: the left and center subfigures demonstrate how increasing $\eta$ leads to greater deviation from the unbiased selection depicted in the right subfigure.

To empirically evaluate our approach, we conducted simulations on both MNIST and FashionMNIST datasets using a \ac{cnn} architecture in Tables \ref{table: test acc 1} and \ref{table: test acc 2}. The \ac{cnn} consists of two convolutional layers (with 16 and 32 channels, each followed by $2 \times 2$ max pooling) and two fully connected layers (with 512 and 32 units, respectively), all using ReLU activation. 
In each training round, we select $N^{\text{select}} = 10$ clients, assign privacy budgets $\epsilon_k$ sampled from $\mathcal{U}(0,1)$, fix $\delta_k$ to $10^{-5}$, and use a batch size $B_k = 128$ for each client. 
The data similarity parameter $s$ is varied from $0\%$ to $100\%$ to assess performance under different degrees of non-IID data.

Figure~\ref{fig: test_acc} presents a comparison of test accuracy for the proposed privacy-aware selection strategy against unbiased and biased client selection baselines. 
In these experiments, clients' privacy budgets are also sampled from a Gaussian mixture distribution, $0.3 \mathcal{N}(0.5, 0.04) + 0.7 \mathcal{N}(10, 1)$, with any negative values discarded. The CNN architecture follows the well-tuned design from~\cite{tramer2020differentially}.

Our results, summarized in Tables~\ref{table: test acc 1} and~\ref{table: test acc 2}, as well as Figure~\ref{fig: test_acc}, clearly demonstrate that \emph{the privacy-aware client selection strategy achieves superior performance across all datasets---MNIST, FashionMNIST, and CIFAR-10.} 
For instance, on CIFAR-10 with $30\%$ data similarity, our method achieves a test accuracy of $32.35\%$, whereas the baselines remain below $23\%$. 
This notable improvement is primarily due to the reduced noise variance resulting from our optimal selection policy. 
Furthermore, we observe that test accuracy becomes more volatile as data similarity decreases, which is expected since lower data similarity amplifies the stochastic gradient variance in \ac{fl}.

\subsection{Effects of Parameters}\label{subsec: parameters}

Next we investigate how key parameters---the subsampling ratio and privacy budgets ($\epsilon$)---affect model performance. Throughout these experiments, we fix the data similarity at 30\%, while all other settings follow those described in Section~\ref{subsec: training performance}.  

\begin{figure}
	\centering
	\includegraphics[width=0.9\linewidth]{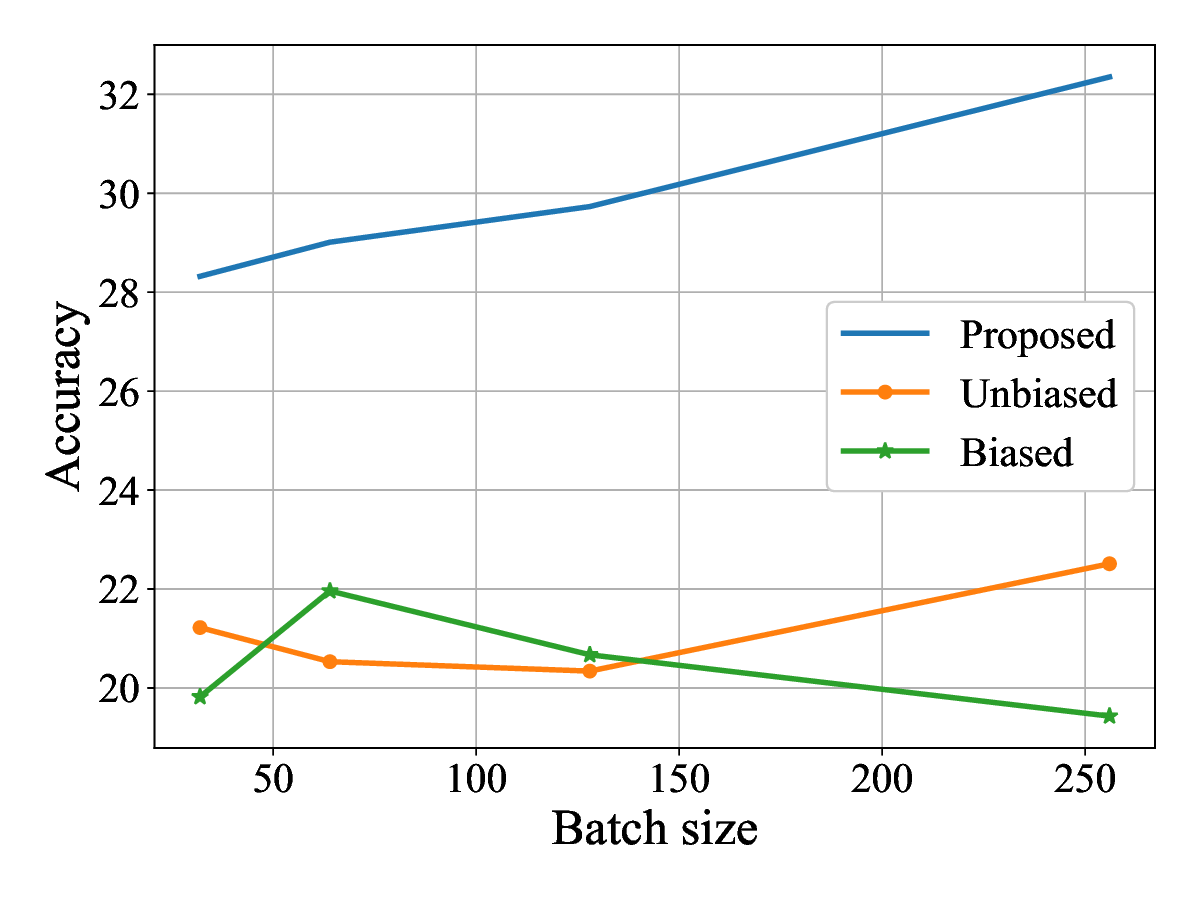}
	\caption{Test accuracy versus batch size.}
	\label{fig: batchsize}
\end{figure}

\begin{figure}
	\centering
	\includegraphics[width=0.9\linewidth]{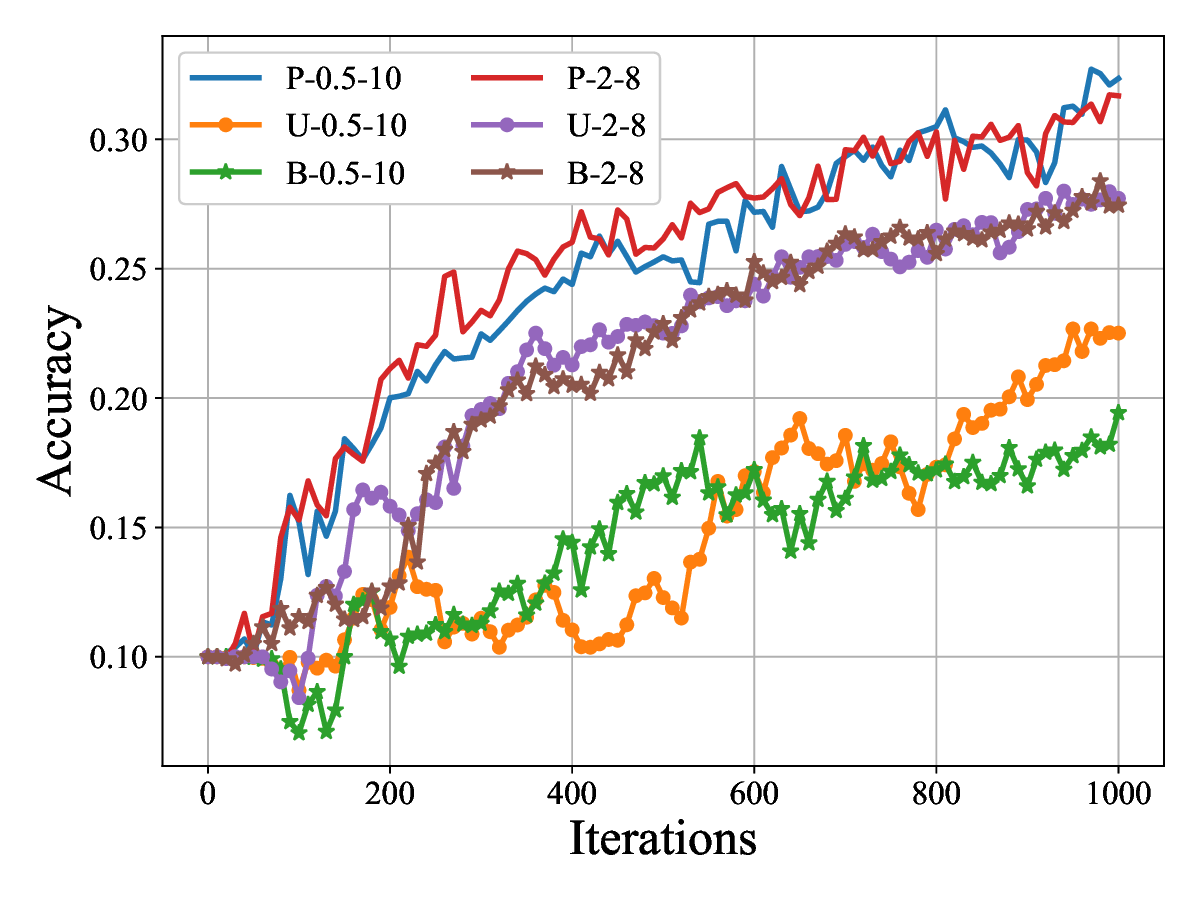}
	\caption{Test accuracy versus privacy budget distributions. "P-$i$-$j$",  "U-$i$-$j$", and "B-$i$-$j$" refer to the proposed, unbiased and biased client selection approaches with privacy budgets following Gaussian mixture distribution $0.3\mathcal{N}(i,0.04) + 0.7\mathcal{N}(j,1)$.}
	\label{fig: budgets}
\end{figure}

\subsubsection{Subsampling} 
Figure~\ref{fig: batchsize} illustrates the effect of varying the subsampling ratio, controlled via the batch size $B_k \in \{32, 64, 128, 256\}$, on the test accuracy. 
As expected, increasing the batch size reduces the added noise per iteration, resulting in higher test accuracy. 
\emph{Importantly, across all batch sizes, our privacy-aware selection strategy consistently outperforms the baseline methods.}

\subsubsection{Privacy Budgets}
Figure~\ref{fig: budgets} examines how the distribution of client privacy budgets influences performance on CIFAR-10. 
We consider two scenarios, where privacy budgets ($\epsilon$) are sampled from different truncated Gaussian mixture distributions:
\begin{itemize}
	\item \underline{Case 1: Small and Dispersed Budgets} \\
	$0.3 \cdot \mathcal{N}(0.5, 0.04) + 0.7 \cdot \mathcal{N}(10, 1)$
	\item \underline{Case 2: Large and Concentrated Budgets} \\
	$0.3 \cdot \mathcal{N}(2, 0.04) + 0.7 \cdot \mathcal{N}(8, 1)$
\end{itemize}

The results demonstrate that the performance improvement from our strategy is strongly influenced by the distribution of privacy budgets. When many clients have small privacy budgets (Case 1), privacy-aware client selection yields a substantial performance gain of approximately 10\%. In contrast, when privacy budgets are generally larger and more concentrated (Case 2), the improvement is more modest (about 4\%).

\section{Conclusions}\label{sec: conclusions}
In this work, we addressed the challenge of client selection in \ac{dp}-\ac{fl} under heterogeneous privacy budgets and non-IID data. 
We proposed a novel privacy-aware client selection strategy, supported by a rigorous convergence analysis and a closed-form expression for the irreducible training loss. 
Our convex optimization approach effectively mitigates the amplified training loss caused by privacy and data heterogeneity, leading to substantial improvements in test accuracy over existing baselines particularly when privacy budgets are limited and highly diverse.

Future research will focus on designing adaptive client selection strategies that respond to real-time changes in client privacy requirements and data distributions. We also aim to extend our framework to support more complex federated learning scenarios, such as personalized models and multi-modal applications, and to explore integration with advanced privacy accounting and resource-aware optimization methods.

\bibliographystyle{IEEEtran}
\bibliography{ref}

\clearpage
\appendix

\subsection{Proof of Theorem 1}\label{sec: proof_convergence_convex}
We first introduce the preliminary key lemmas for proving Theorem 1 in Appendix \ref{sec: key lemmas}.
We provide the complete proof of Theorem 1 in Appendix \ref{sec: complete Th1}, and defer the proof of the Lemmas \ref{lemma: A3}, \ref{lemma: A4}, \ref{lemma: A5} to Appendix \ref{sec: proof lemma 4}, \ref{sec: proof lemma 5}, \ref{sec: proof lemma 6}.

\subsubsection{Key lemmas}\label{sec: key lemmas}
We present the key lemmas here and prove them later.
\begin{lemma}\label{lemma: A3}
	Suppose that Assumptions 1-3 hold. Choose an $\Theta(\frac{1}{t})$ stepsize $\oldfrac{\beta}{\gamma+t}$ $(\beta>\oldfrac{B_1}{C\mu}, \gamma > 1)$. 
	We have
	\begin{equation}\label{equ: lemma3}
		\begin{aligned}
			&\mathbb{E}\left[F(\boldsymbol{w}_{T})\!-\!F^*\right]\le\frac{\gamma}{\gamma+T}\left(F(\boldsymbol{w}_{1})-F^*\right)\\
			&\!+\frac{1}{2}\frac{\beta^2}{(\gamma+T)^2}L\sum_{t=1}^{T}\left(\sum_{k=1}^{N}(p^\textnormal{s}_{k})^2DTV_kC^2+C^2\right)\\
			&+\!\frac{\beta}{\gamma+T} \underbrace{\sum_{t=1}^{T}\mathbb{E}\left[|\nabla F(\tilde{\boldsymbol{w}}_{t})^\top\left( \boldsymbol{g}^\textnormal{clip}_{t}\!-\!\nabla F^\textnormal{clip}(\tilde{\boldsymbol{w}}_{t}) \right)|\right]}_{A_1}.
		\end{aligned}
	\end{equation}
\end{lemma} 

\begin{lemma}\label{lemma: A4}
	Suppose that Assumptions 1, 3 hold. We have
	\begin{equation}\label{equ: lemma4}
		\begin{aligned}
			A_1 =& 
			\sum_{t=1}^{T}\mathbb{E}\!\left[|\nabla F(\tilde{\boldsymbol{w}}_{t})^\top\left( \boldsymbol{g}^\textnormal{clip}_{t}\!-\!\nabla F^\textnormal{clip}(\tilde{\boldsymbol{w}}_{t}) \right)|\right]\\ 		
			\le&TLA_2G^\textnormal{select,clip}C\!+\!T\mathbb{E}\left[\left\|\nabla F(\boldsymbol{w}_{T+1})\right\|_2\right]\!G^\textnormal{select,clip}C,
		\end{aligned}
	\end{equation}
	where 
	\begin{align}
		A_2 \!=\!\! \left(\!\frac{1}{T}\!\sum_{t=1}^{T}\!\mathbb{E}\!\left[\left\|\boldsymbol{w}_t-\boldsymbol{w}^{\textnormal{b},*}\right\|_2\right]\!+\!\frac{T\!-\!\!1}{T}\mathbb{E}\!\left\|\boldsymbol{w}_{T+1}\!-\!\boldsymbol{w}^{\textnormal{b},*}\right\|_2\!\!\right)
	\end{align}
\end{lemma}

\begin{lemma}\label{lemma: A5}
	Suppose that Assumptions 1-3 hold. 
	Choose an $\Theta(\frac{1}{t})$ stepsize $\oldfrac{\beta}{\gamma+t}$ $(\beta>\oldfrac{B_1}{C\mu}, \gamma > 1)$.
	We have
	\begin{equation}\label{equ: lemma5}
		\begin{aligned}
			&A_2\le \frac{3}{\sqrt{\gamma+T}}\\
			&\sqrt{\!\gamma\!\left\|\boldsymbol{w}_1\!\!-\!\boldsymbol{w}^{\textnormal{b},*}\!\right\|^2_2\!\!+\!\!\frac{\beta^2T^2}{\gamma\!+\!T}\!\!\left(G^\textnormal{DP}\!+\!\!\frac{1}{T}\!\!\right)\!C^2\!+\!\!\frac{2\beta T\!B_1}{\mu} G^\textnormal{b,clip}C}.
		\end{aligned}
	\end{equation}
\end{lemma}

\subsubsection{Complete Proof of Theorem 1}\label{sec: complete Th1}
We prove Theorem 1 based on the above key lemmas (Lemmas \ref{lemma: A3}, \ref{lemma: A4}, and \ref{lemma: A5}). 
At first, we bound $A_2G^\text{select, clip}C$ based on Lemma \ref{lemma: A5}.
Then, we substitute the bound of $A_2G^\text{select, clip}C$ into Lemmas \ref{lemma: A3} and \ref{lemma: A4} to obtain a complete inequality of the optimality gap.
Finally, solving the inequality yields Theorem 1.

\begin{proof}
	Based on Lemma \ref{lemma: A5}, we can bound $A_2G^\text{select,clip}C$ by the arithmetic mean-geometric mean inequality. 
	That is, for every $\Gamma > 0$,
	\begin{equation}\label{equ: am-gm}
		\begin{aligned}
			&A_2 G^\text{select,clip}C \!\le\! \frac{1}{2}\frac{\beta T}{\gamma\!+\!T}\!\left[\Gamma\frac{(\gamma+T)^2}{\beta^2T^2} A_2^2\!+\!\frac{1}{\Gamma}(G^\text{select,clip})^2C^2\right] \\
			\le&\frac{1}{2}\frac{\beta T}{\gamma+T}9\Gamma\left(\frac{\gamma+T}{\beta^2T^2} \gamma\left\|\boldsymbol{w}_1-\boldsymbol{w}^{\text{b},*}\right\|^2_2+G^\text{DP}C^2+\frac{C^2}{T}\right.\\
			&\left.+\frac{2B_1(\gamma+T)}{\beta T \mu}G^\text{b,clip}C\right)+\frac{1}{2}\frac{\beta T}{\gamma+T}\frac{1}{\Gamma}(G^\text{select,clip})^2C^2.
		\end{aligned}
	\end{equation}
	Substituting (\ref{equ: am-gm}) and (\ref{equ: lemma4}) to (\ref{equ: lemma3}), we have
	\begin{equation}\label{equ: quad}
		\begin{aligned}
			&\mathbb{E}\left[F(\boldsymbol{w}_{T+1})-F^*\right]\le \frac{\gamma}{\gamma+T}\left(F(\boldsymbol{w}_{1})-F^*\right)\\
			&+\frac{\beta T}{\gamma+T}\sqrt{2L}\sqrt{\mathbb{E}\left[F(\boldsymbol{w}_{T+1})-F^*\right]}G^\text{select,clip}C\\
			&+\frac{1}{2}\frac{\beta^2 T^2L}{(\gamma+T)^2}9\Gamma\left(\frac{\gamma+T}{\beta^2T^2} \gamma\left\|\boldsymbol{w}_1-\boldsymbol{w}^{\text{b},*}\right\|^2_2+G^\text{DP}C^2\right.\\
			&\left.+\frac{C^2}{T}+\frac{2B_1(\gamma+T)}{\beta T \mu}G^\text{b,clip}C\right)+\frac{1}{2}\frac{\beta^2T^2L}{(\gamma+T)^2}G^\text{DP}C^2\\
			&+\frac{1}{2}\frac{\beta^2T}{(\gamma+T)^2}LC^2+\frac{1}{2}\frac{\beta^2 T^2L}{(\gamma+T)^2}\frac{1}{\Gamma}(G^\text{select,clip})^2C^2.
		\end{aligned}
	\end{equation}
	This inequality (\ref{equ: quad}) is an quadratic inequality of $\sqrt{\mathbb{E}\left[F(\boldsymbol{w}_{T+1})-F^*\right]}$.
	Through solving (\ref{equ: quad}), we obtain the following upper bound of $\sqrt{\mathbb{E}\left[F(\boldsymbol{w}_{T+1})-F^*\right]}$ as in (\ref{equ: solving}).
	\begin{figure*}
		\begin{equation}\label{equ: solving}
			\begin{aligned}
				&\sqrt{\mathbb{E}\left[F(\boldsymbol{w}_{T+1})\!-\!F^*\right]}\le\frac{1}{2}\left[\frac{\beta T}{\gamma+T}\sqrt{2L}G^\text{select,clip}C\right]\!+\!\frac{\sqrt{2}}{2}C\sqrt{\!\!\frac{\beta^2 T^2L}{(\gamma\!+\!T)^2}(1\!+\!\frac{1}{\Gamma})(G^\text{select,clip})^2\!\!+\!\!\frac{\beta^2 T^2\!L}{(\gamma\!+\!T)^2}(9\Gamma\!+\!1\!)G^\text{DP}}\\
				&+\sqrt{\frac{1}{\gamma+T}}\sqrt{\frac{9}{2}\!L\Gamma\gamma\!\left\|\boldsymbol{w}_1\!\!-\!\boldsymbol{w}^{\text{b},*}\!\right\|^2_2\!\!+\!\!\frac{9\Gamma\!+\!1}{2}\frac{\beta^2T\!L}{\gamma\!+\!T}C^2\!+\!18\Gamma\beta TL\frac{B_1}{\mu}G^\text{b,clip}C\!+\!\gamma\!\left(F(\boldsymbol{w}_1)\!-\!F^*\!\right)}\\
				\le&\frac{1}{2}\frac{\sqrt{2L}\beta TC}{\gamma+T}\left[G^\text{select,clip}\!+\!\!\sqrt{\!\left(1\!+\!\frac{1}{\Gamma}\right)\!(G^\text{select,clip})^2\!\!+\!(9\Gamma+1)G^\text{DP}}\right]+\!\sqrt{\!18\Gamma\frac{\beta TL}{\gamma+T}\frac{B_1}{\mu}G^\text{b,clip}C}\\
				&+\!\sqrt{\frac{1}{2}\frac{9L\Gamma}{\gamma\!+\!T}\gamma\left\|\boldsymbol{w}_1\!-\!\boldsymbol{w}^{\text{b},*}\right\|^2_2\!+\!\frac{9\Gamma+1}{2}\frac{\beta^2TL}{(\gamma\!+\!T)^2}C^2\!+\frac{\gamma}{\gamma+T}\left(F(\boldsymbol{w}_1)-F^*\right)}.
			\end{aligned}
		\end{equation}
		\rule[0pt]\linewidth{0.03cm}
	\end{figure*}
\end{proof}

\subsubsection{Proof of Lemma \ref{lemma: A3}}\label{sec: proof lemma 4}
\begin{proof}
	By the property of strong convexity (Assumption 2), \textit{i.e.},
	$2\mu(F(\boldsymbol{w})-F^*) \le ||\nabla F(\boldsymbol{w})||^2_2$ and $(\ref{equ: descent})$, we have
	\begin{equation}\label{equ: one-step1}
		\begin{aligned}
			&\mathbb{E}\left[F(\tilde{\boldsymbol{w}}_{t+1})\right]-F(\tilde{\boldsymbol{w}}_{t}) \le -\alpha_{t}\frac{2C\mu}{B_1}\left(F(\tilde{\boldsymbol{w}}_{t})-F^*\right)\\
			&-\alpha_{t} \nabla F(\tilde{\boldsymbol{w}}_{t})^\top\left( \boldsymbol{g}^\text{clip}_{t}-\nabla F^\text{clip}(\tilde{\boldsymbol{w}}_{t}) \right)\\
			&+\frac{1}{2}\alpha_{t}^2L\left(\mathbb{E}\left[\left\|\boldsymbol{n}_{t}\right\|^2_2\right]+C^2\right).
		\end{aligned}
	\end{equation} 
	Adding $F(\tilde{\boldsymbol{w}}_{t}) -F^*$ on both sides of (\ref{equ: one-step1}), we have
	\begin{equation}
		\begin{aligned}
			&\mathbb{E}\left[F(\boldsymbol{w}_{t+1})\right]\!-\!F^* \le \left(1\!-\!\alpha_{t}\frac{2C\mu}{B_1}\right)\left(F(\tilde{\boldsymbol{w}}_{t})-F^*\right)\!\\
			&-\!\alpha_{t} \nabla F(\tilde{\boldsymbol{w}}_{t})^\top\left( \boldsymbol{g}^\text{clip}_{t}-\nabla F^\text{clip}(\tilde{\boldsymbol{w}}_{t}) \right)\\
			&+\frac{1}{2}\alpha_{t}^2L\left(\mathbb{E}\left[\left\|\boldsymbol{n}_{t}\right\|^2_2\right]+C^2\right).
		\end{aligned}
	\end{equation}
	Setting the stepsize as $\alpha_{t} = \oldfrac{\beta}{\gamma+t}$, we have 
	\begin{equation}\label{equ: one-step}
		\begin{aligned}
			&\mathbb{E}\left[F(\boldsymbol{w}_{t+1})\right]\!-\!F^* \le \left(1\!-\!\frac{\beta}{\gamma+t}\frac{2C\mu}{B_1}\right)\left(F(\tilde{\boldsymbol{w}}_t)-F^*\right)\!\\
			&-\!\frac{\beta}{\gamma+t} \nabla F(\tilde{\boldsymbol{w}}_{t})^\top\left( \boldsymbol{g}^\text{clip}_{t}-\nabla F^\text{clip}(\tilde{\boldsymbol{w}}_{t}) \right)\\
			&+\frac{1}{2}\frac{\beta^2}{(\gamma+t)^2}L\left(\mathbb{E}\left[\left\|\boldsymbol{n}_{t}\right\|^2_2\right]+C^2\right).
		\end{aligned}
	\end{equation}
	
	When $\beta \ge \oldfrac{B_1}{C\mu}$, we have the following inequalities, 
	\begin{equation}\label{inequality: learning rate 1}
		\begin{aligned}
			&\left(1\!-\!\frac{\beta}{\gamma+t}\frac{2C\mu}{B_1}\right)\left(\frac{\beta}{\gamma+t-1}\right)\\
			=&
			\left(1-\frac{1}{\gamma + \tilde{t}}+\frac{1-\oldfrac{2\beta C\mu}{B_1}}{\gamma + t}\right)\left(\frac{\beta}{\gamma+t-1}\right)\\
			=& \frac{\beta}{\gamma+t}+\frac{\beta\left(1-\oldfrac{2\beta C\mu}{B_1}\right)}{(\gamma+t)(\gamma+t-1)}\\\le& \frac{\beta}{\gamma+t},\\
		\end{aligned}
	\end{equation}
	and
	\begin{equation}\label{inequality: learning rate 2}
		\begin{aligned}	
			&\left(1\!-\!\frac{\beta}{\gamma+t}\frac{2C\mu}{B_1}\right)\left(\frac{\beta}{\gamma+t-1}\right)^2\\
			=&
			\left(1-\frac{1}{\gamma + t}+\frac{1-\frac{2\beta C\mu}{B_1}}{\gamma + t}\right)\left(\frac{\beta}{\gamma+t-1}\right)^2\\
			=&
			\frac{\beta^2}{(\gamma+t)(\gamma+t-1)}+\frac{\beta^2(1-\frac{2\beta C\mu}{B_1})}{(\gamma+t)(\gamma+t-1)^2}\\
			=& \frac{\beta^2}{(\gamma+t)^2}+\frac{\beta^2}{(\gamma+t)^2(\gamma+t-1)}+\frac{\beta^2(1-\frac{2\beta C\mu}{B_1})}{(\gamma+t)(\gamma+t-1)^2}\\
			\le& \frac{\beta^2}{(\gamma+t)^2}.
		\end{aligned}
	\end{equation}
	Taking expectations over all iterations of (\ref{equ: one-step}) yields
	\begin{equation}\label{equ: all iterations}
		\begin{aligned}
			&\mathbb{E}\left[F(\boldsymbol{w}_{T+1})\!-\!F^*\right]\le\\
			&\prod_{t=1}^{T}\left(1\!-\!\frac{\beta}{\gamma+t}\frac{2C\mu}{B_1}\right)\left(F(\boldsymbol{w}_1)-F^*\right)+ \!\sum_{t=0}^{T-1}\!\frac{\beta}{\gamma\!+\!t}\prod_{i=0}^{T-t-1}\!\!\\
			&\!(1 \!-\! \frac{\beta}{\gamma\!+\!T\!-\!i}\frac{2C\mu}{B_1})\mathbb{E}\!\left[\!|\nabla F(\tilde{\boldsymbol{w}}_{t})^\top\!\left( \boldsymbol{g}^\text{clip}_{t}\!-\!\nabla F^\text{clip}(\tilde{\boldsymbol{w}}_{t}) \right)|\right]\\
			&\!+\!\frac{1}{2}\sum_{t=1}^{T}\!\frac{\beta^2}{(\gamma+t)^2}\!\!\prod_{i=0}^{T-t-1}\!\!\left(\!1\! -\! \frac{\beta}{\gamma\!+\!T\!-\!i}\frac{2C\mu}{B_1}\right)\!L\!\left(\!\mathbb{E}\left[\!\left\|\boldsymbol{n}_{t}\right\|^2_2\right]\!+\!C^2\!\right)\!.
		\end{aligned}
	\end{equation}
	Substituting (\ref{inequality: learning rate 1}) and (\ref{inequality: learning rate 2}) to (\ref{equ: all iterations}), we have
	\begin{equation}
		\begin{aligned}
			&\mathbb{E}\left[F(\boldsymbol{w}_{T+1})\!-\!F^*\right]\le \frac{\gamma}{\gamma+T}\left(F(\boldsymbol{w}_{1})-F^*\right)\!\\
			&+\!\frac{\beta}{\gamma+T} \underbrace{\sum_{t=1}^{T}\mathbb{E}\left[\left|\nabla F(\tilde{\boldsymbol{w}}_{t})^\top\left( \boldsymbol{g}^\text{clip}_{t}\!-\!\nabla F^\text{clip}(\tilde{\boldsymbol{w}}_{t}) \right)\right|\right]}_{A_1}\\
			&+\frac{1}{2}\frac{\beta^2}{(\gamma+T)^2}L\sum_{t=1}^{T}\left(\mathbb{E}\left[\left\|\boldsymbol{n}_{t}\right\|^2_2\right]+C^2\right).
		\end{aligned}
	\end{equation}
	Recall that when the clients have diverse dataset sizes $\{|\mathcal{M}_k|\}_{k\in\mathcal{N}}$, privacy budgets $\{(\epsilon_k,\delta_k)\}_{k\in\mathcal{N}}$, and sampling rates $\{r_k\}_{k\in\mathcal{N}}$, the noise variance is
	\begin{equation}\label{equ: noise}
		\begin{aligned}
			&\mathbb{E}\left[\left\|\boldsymbol{n}_{t}\right\|^2_2\right] = \mathbb{V}\left[\boldsymbol{n}_{t}\right] = \frac{1}{N^\text{select}}\sum_{k=1}^{N}(p^\text{s}_{k})^2DTN^\text{select}V_kC^2 \\
			&= \! \sum_{k=1}^{N}(p^\text{s}_{k})^2TD\oldfrac{8\log\left(e\!+\!\left(\oldfrac{r_k\log\left(1\!+\!\oldfrac{1}{r_k}(e^{\epsilon_k}-1)\right)}{\delta_k}\right)\right)}{|\mathcal{M}_k|^2r_k^2\log\left(1+\oldfrac{1}{r_k}(e^{\epsilon_k}-1)\right)^2}C^2.
		\end{aligned}
	\end{equation}
	It follows that
	\begin{equation}\label{inequality: loss}
		\begin{aligned}
			&\mathbb{E}\left[F(\boldsymbol{w}_{T+1})\!-\!F^*\right]\le \frac{\gamma}{\gamma+T}\left(F(\boldsymbol{w}_{1})-F^*\right)\!\\
			&+\!\frac{\beta}{\gamma+T} \underbrace{\sum_{t=1}^{T}\mathbb{E}\left[|\nabla F(\tilde{\boldsymbol{w}}_{t})^\top\left( \boldsymbol{g}^\text{clip}_{t}\!-\!\nabla F^\text{clip}(\tilde{\boldsymbol{w}}_{t}) \right)|\right]}_{A_1}\\
			&+\frac{1}{2}\frac{\beta^2}{(\gamma+T)^2}L\sum_{t=1}^{T}\left(\sum_{k=1}^{N}(p^\text{s}_{k})^2DTV_kC^2+C^2\right).
		\end{aligned}
	\end{equation}
\end{proof}

\subsubsection{Proof of Lemma \ref{lemma: A4}}\label{sec: proof lemma 5}
\begin{proof}
	Based on Cauchy–Schwarz inequality, we have
	\begin{equation}\label{equ: A_1}
		\begin{aligned}
			A_1 =&\sum_{t=1}^{T}\mathbb{E}\left[|\nabla F(\tilde{\boldsymbol{w}}_{t})^\top\left( \boldsymbol{g}^\text{clip}_{t}\!-\!\nabla F^\text{clip}(\tilde{\boldsymbol{w}}_{t}) \right)|\right]\\
			\le& \sum_{t=1}^{T}\mathbb{E}\left[\left\|\nabla F(\tilde{\boldsymbol{w}}_{t})\right\|_2\left\| \boldsymbol{g}^\text{clip}_{t}-\nabla F^\text{clip}(\tilde{\boldsymbol{w}}_{t}) \right\|_2\right].
		\end{aligned}
	\end{equation}
	We can bound $\left\| \boldsymbol{g}^\text{clip}_{\tilde{t}}-\nabla F^\text{clip}(\tilde{\boldsymbol{w}}_{\tilde{t}}) \right\|_2$ by
	\begin{equation}\label{equ: clip}
		\begin{aligned}
			&\left\| \boldsymbol{g}^\text{clip}_{t}-\nabla F^\text{clip}(\tilde{\boldsymbol{w}}_{t}) \right\|_2\\
			\le&\left\| \boldsymbol{g}^\text{clip}(\tilde{\boldsymbol{w}}_{t})\!-\!\nabla F^\text{clip}(\tilde{\boldsymbol{w}}_{t}) \right\|_2 + \left\|\boldsymbol{g}_{t}^\textnormal{clip}-\boldsymbol{g}^\text{clip}(\tilde{\boldsymbol{w}}_{t})\right\|_2\\
			\le&
			\left\|\boldsymbol{g}^\text{clip}(\tilde{\boldsymbol{w}}_{t})\!-\!\boldsymbol{g}^\text{clip,unbiased}(\tilde{\boldsymbol{w}}_{t})\right\|_2 + B_2\\
			&+\left\|\boldsymbol{g}^\text{clip,unbiased}(\tilde{\boldsymbol{w}}_{t})\!-\!\nabla F^\text{clip}(\tilde{\boldsymbol{w}}_{t})\right\|_2\\
			\le& \left\|\boldsymbol{p}^\text{s}-\boldsymbol{p}^\text{u}\right\|_1C+\Delta_t C + B_2.
		\end{aligned}
	\end{equation}
	Then, substituting (\ref{equ: clip}) into (\ref{equ: A_1}), we have
	\begin{equation}\label{equ: A_2}
		\begin{aligned}
			&A_1\le \\
			&\! \sum_{t=1}^{T}\!\mathbb{E}\left[\left\|\nabla F(\tilde{\boldsymbol{w}}_{t})\right\|_2\right]\!\underbrace{\left(\left\| \boldsymbol{p}^\textnormal{s}\!-\!\boldsymbol{p}^\textnormal{u} \right\|_1\!+\!\max_t\{\left\|\Delta_t\right\|_1\}\!+\!\frac{B_2}{C}\right)}_{G^\text{select,clip}}\!C \\
			=&T\!\!\left(\!\mathbb{E}\left[\frac{1}{T}\sum_{t=1}^{T}\!\left\|\nabla F(\tilde{\boldsymbol{w}}_{t})\right\|_2\!-\!\left\|\nabla F(\boldsymbol{w}_{T+1})\right\|_2\right]\!\right)\\
			&\!G^\text{select,clip}C\!+\!T\mathbb{E}\left[\left\|\nabla F(\boldsymbol{w}_{T+1})\right\|_2\right]G^\text{select,clip}C\\
			\le&T\!\left(\mathbb{E}\left[\frac{1}{T}\sum_{t=1}^{T}\left\|\nabla F(\tilde{\boldsymbol{w}}_{t})\!-\!\nabla F(\boldsymbol{w}_{T+1})\right\|_2\right]\right)\\
			\!&G^\text{select,clip}C+T\mathbb{E}\left[\left\|\nabla F(\boldsymbol{w}_{T+1})\right\|_2\right]G^\text{select,clip}C\\
			\le&T\!\left(\!\frac{1}{T}\sum_{t=1}^{T}L\mathbb{E}\left[\left\|\tilde{\boldsymbol{w}}_{t}-\boldsymbol{w}_{T+1}\right\|_2\right]\!\right)G^\text{select,clip}C\!\\
			&+\!T\mathbb{E}\left[\left\|\nabla F(\boldsymbol{w}_{T+1})\right\|_2\right]G^\text{select,clip}C\\
			\le&T\!L\!\left(\underbrace{\!\frac{1}{T}\sum_{\tilde{t}=1}^{T}\mathbb{E}\!\left[\left\|\tilde{\boldsymbol{w}}_{\tilde{t}}\!-\!\boldsymbol{w}^{\text{b},*}\right\|_2\right]\!+\frac{T\!-\!1}{T}\mathbb{E}\left[\left\|\boldsymbol{w}_{T+1}\!-\!\boldsymbol{w}^{\text{b},*}\right\|\right]}_{A_2}\!\right)\\
			&G^\text{select,clip}C\!+\!T\mathbb{E}\left[\left\|\nabla F(\boldsymbol{w}_{T+1}\!)\right\|_2\right]\!G^\text{select,clip}C\!.
		\end{aligned}
	\end{equation}
\end{proof}

\subsubsection{Proof of Lemma \ref{lemma: A5}}\label{sec: proof lemma 6}
\begin{proof}
	To bound $A_2$, we need to bound $\mathbb{E}\left[\left\|\tilde{\boldsymbol{w}}_{t+1}-\boldsymbol{w}^{\text{b},*}\right\|_2^2\right]$ first.
	Taking expectation over batch, selection, and noise in iteration $t$, we have
	\begin{equation}
		\begin{aligned}
			&\mathbb{E}\left[\left\|\tilde{\boldsymbol{w}}_{t+1}-\boldsymbol{w}^{\text{b},*}\right\|_2^2\right] \\
			=&\mathbb{E}\left[\left\|\tilde{\boldsymbol{w}}_{t}-\alpha_{t}\boldsymbol{g}_{t}^\text{clip}-\boldsymbol{w}^{\text{b},*}-\alpha_{t}\boldsymbol{n}_{t}\right\|_2^2\right]\\
			=& \left\|\tilde{\boldsymbol{w}}_{t}-\boldsymbol{w}^{\text{b},*}\right\|^2_2-2\alpha_{t}\mathbb{E}\left[\langle \tilde{\boldsymbol{w}}_{t}-\boldsymbol{w}^{\text{b},*}, \tilde{\boldsymbol{g}}^\text{clip}_{t}\rangle\right]\\
			&+\alpha_{t}^2\mathbb{E}\left[\left\|\tilde{\boldsymbol{g}}^\text{clip}_{t}+\boldsymbol{n}_{t}\right\|^2_2\right]\\
			=& \left\|\tilde{\boldsymbol{w}}_{t}-\boldsymbol{w}^{\text{b},*}\right\|^2_2\!-2\alpha_{t}\langle\tilde{\boldsymbol{w}}_{t} - \boldsymbol{w}^{b,*}, \boldsymbol{g}^\textnormal{clip}_{t}- \boldsymbol{g}^\text{clip}(\tilde{\boldsymbol{w}}_{t})\rangle\\
			&-2\alpha_{t}\!\left\langle \tilde{\boldsymbol{w}}_{t}\!-\!\boldsymbol{w}^{\text{b},*}, \boldsymbol{g}^\text{clip}(\boldsymbol{w}_t)\!-\!\nabla F^\text{b,clip}(\boldsymbol{w}_t)\right\rangle\\
			&-2\alpha_{t}\langle \boldsymbol{w}_t-\boldsymbol{w}^{\text{b},*}, \nabla F^\text{b,clip}(\boldsymbol{w}_t)\rangle
			+\alpha_{t}^2\mathbb{E}\left[\left\|\tilde{\boldsymbol{g}}^\text{clip}(\boldsymbol{w}_t)\right\|^2_2\right]\\
			&+\alpha_{t}^2\mathbb{E}\left[\left\|\boldsymbol{n}_{t}\right\|_2^2\right]\\
			\le&\left\|\tilde{\boldsymbol{w}}_{t}-\boldsymbol{w}^{\text{b},*}\right\|^2_2-2\frac{C}{B_1}\alpha_{t}\mu\left\|\tilde{\boldsymbol{w}}_{t}-\boldsymbol{w}^{\text{b},*}\right\|_2^2+\alpha_{t}^2C^2\\
			&+\alpha_{t}^2\mathbb{E}\left[\left\|\boldsymbol{n}_{t}\right\|_2^2\right]\!-\!2\alpha_{t}\!\left\langle\! \tilde{\boldsymbol{w}}_{t}\!-\!\boldsymbol{w}^{\text{b},*}, \boldsymbol{g}^\text{clip}(\tilde{\boldsymbol{w}}_{t})\!-\!\nabla F^\text{b,clip}(\tilde{\boldsymbol{w}}_{t})\!\right\rangle\\
			&-2\alpha_{t}\langle\tilde{\boldsymbol{w}}_{t}\!-\!\boldsymbol{w}^{\text{b},*}, \boldsymbol{g}^\text{clip}_{t} - \boldsymbol{g}^\text{clip}(\tilde{\boldsymbol{w}}_{t})\rangle\\
			=&\left(1-\frac{2C\mu}{B_1}\alpha_{t}\right)\left\|\tilde{\boldsymbol{w}}_{t}-\boldsymbol{w}^{\text{b},*}\right\|^2_2+\alpha_{t}^2C^2+\alpha_{t}^2\mathbb{E}\left[\left\|\boldsymbol{n}_{t}\right\|_2^2\right]\\
			&-\!2\alpha_{t}\!\left\langle\! \tilde{\boldsymbol{w}}_{t}\!-\!\boldsymbol{w}^{\text{b},*}, \boldsymbol{g}^\text{clip}(\tilde{\boldsymbol{w}}_{t})\!-\!\nabla F^\text{b,clip}(\tilde{\boldsymbol{w}}_{t})\!\right\rangle\\
			&-2\alpha_{t}\langle\tilde{\boldsymbol{w}}_{t}\!-\!\boldsymbol{w}^{\text{b},*}, \boldsymbol{g}^\text{clip}_{t} - \boldsymbol{g}^\text{clip}(\tilde{\boldsymbol{w}}_{t})\rangle\\
			\le&\left(1-\frac{2C\mu}{B_1}\alpha_{t}\right)\left\|\tilde{\boldsymbol{w}}_{t}-\boldsymbol{w}^{\text{b},*}\right\|^2_2+\alpha_{t}^2C^2+\alpha_{t}^2\mathbb{E}\left[\left\|\boldsymbol{n}_{t}\right\|_2^2\right]\\
			&+2\alpha_{t}\frac{B_1}{\mu} \underbrace{\left(\max_t\{ \left\|\Delta^b_t\right\|_1\}+\frac{B_2}{C}\right)}_{G^\text{b,clip}}C,
		\end{aligned}
	\end{equation}
	where $\nabla F^\text{b,clip}(\tilde{\boldsymbol{w}}_{t})$ is the clipped version of $\nabla F^\text{b}(\tilde{\boldsymbol{w}}_{t})$.
	Taking expectation over all iterations and applying (\ref{inequality: learning rate 1}) and (\ref{inequality: learning rate 2}), we have
	\begin{equation}\label{equ: 76}
		\begin{aligned}
			\mathbb{E}\left[\left\|\tilde{\boldsymbol{w}}_{t+1}\!-\!\boldsymbol{w}^{\text{b},*}\right\|_2^2\right] \le& \frac{\gamma}{\gamma+t} \left\|\boldsymbol{w}_1\!-\!\boldsymbol{w}^{\text{b},*}\right\|^2_2\\
			&+\frac{\beta^2t}{(\gamma+t)^2}\left(TG^\text{DP}C^2+C^2\right)\\
			&+2\frac{\beta t}{\gamma+t}\frac{B_1}{\mu}G^\text{b,clip}C.
		\end{aligned}
	\end{equation}
	Taking the square roots on both sides of (\ref{equ: 76}), we have
	\begin{equation}
		\begin{aligned}
			&\sqrt{\mathbb{E}\left[\left\|\tilde{\boldsymbol{w}}_{t+1}\!-\!\boldsymbol{w}^{\text{b},*}\right\|_2^2\right]} \!\le\sqrt{\frac{1}{\gamma+t}}\\
			&\sqrt{\!\gamma\!\left\|\boldsymbol{w}_1\!\!-\!\boldsymbol{w}^{\text{b},*}\right\|^2_2\!\!+\!\!\frac{\beta^2tT}{\gamma\!+t}\!\left(G^\text{DP}\!\!+\!\!\frac{1}{T}\!\!\right)\!C^2\!+\!\!\frac{2\beta tB_1}{\mu}G^\text{b,clip}C}.
		\end{aligned}
	\end{equation}
	Because $\frac{t}{(\gamma + t)}$ increases with $t$, we have
	\begin{equation}\label{equ: avg1}
		\begin{aligned}
			&\sqrt{\mathbb{E}\left[\left\|\tilde{\boldsymbol{w}}_{t+1}\!-\!\boldsymbol{w}^{\text{b},*}\right\|_2^2\right]}	\le \sqrt{\!\frac{1}{\gamma\!+\!t}}\\
			&\sqrt{\!\gamma\!\left\|\boldsymbol{w}_1\!\!-\!\boldsymbol{w}^{\text{b},*}\!\right\|^2_2\!\!+\!\!\frac{(\beta T)^2}{\gamma\!\!+\!T}\!\!\left(\!G^\text{DP}\!+\!\!\frac{1}{T}\!\!\right)\!\!C^2\!+\!E_1},
		\end{aligned}
	\end{equation}
	where $E_1 = (\frac{2\beta T\!B_1}{\mu}) G^\text{b,clip}C$.
	In addition, we have the following inequality for the time average,
	\begin{equation}\label{equ: square inequ}
		\begin{aligned}
			&\frac{1}{T}\sum_{t=1}^{T+1}\sqrt{\frac{1}{\gamma+t-1}}\\
			\le&\frac{2}{T}\left(\sqrt{\gamma+T}-\sqrt{\gamma-1}\right)+\frac{1}{T}\\
			\le& \frac{2}{T}\frac{T}{\sqrt{\gamma+T}+\sqrt{\gamma-1}}\\
			\le& \frac{2}{\sqrt{\gamma+T}}.
		\end{aligned}
	\end{equation}
	Taking the average of all iterations on (\ref{equ: avg1}) yields
	\begin{equation}\label{equ: avg2}
		\begin{aligned}
			&\frac{1}{T}\!\sum_{t=1}^{T+1}\!\sqrt{\mathbb{E}\!\left[\left\|\tilde{\boldsymbol{w}} _{t}-\boldsymbol{w}^{\text{b},*}\right\|_2^2\right]}\\
			&\le \sum_{t=1}^{T+1}\sqrt{\!\frac{1}{\gamma\!+\!t-1}}\\
			&\sqrt{\!\gamma\!\left\|\boldsymbol{w}_1\!-\!\boldsymbol{w}^{\text{b},*}\!\right\|^2_2\!+\!\!\frac{\beta^2(T)^2}{\gamma\!+\!T}\!\!\left(\!G^\text{DP}\!+\!\!\frac{1}{T}\!\!\right)\!\!C^2\!+\!\!E_1}\!.
		\end{aligned}
	\end{equation}
	Substituting (\ref{equ: square inequ}) to (\ref{equ: avg2}), it follows that
	\begin{equation}\label{equ: avg3}
		\begin{aligned}
			&\frac{1}{T}\!\sum_{t=1}^{T+1}\!\sqrt{\mathbb{E}\!\left[\left\|\tilde{\boldsymbol{w}}_{t}-\boldsymbol{w}^{\text{b},*}\right\|_2^2\right]}\!\le \frac{2}{\sqrt{\gamma\!+\!T}}\\ &\sqrt{\!\gamma\!\left\|\boldsymbol{w}_1\!\!-\!\boldsymbol{w}^{\text{b},*}\!\right\|^2_2\!\!+\!\!\frac{\beta^2(T)^2}{\gamma\!+\!T}\!\!\left(\!G^\text{DP}\!+\!\!\frac{1}{T}\!\!\right)\!\!C^2\!\!+\!\!E_1}.
		\end{aligned}
	\end{equation}
	Therefore, based on (\ref{equ: avg1}) and (\ref{equ: avg3}), we can bound $A_2$ as
	\begin{equation}\label{equ: A_2_final}
		\begin{aligned}
			&A_2\le\frac{2}{\sqrt{\gamma+T}}\\
			&\sqrt{\!\gamma\!\left\|\boldsymbol{w}_1\!-\!\boldsymbol{w}^{\text{b},*}\!\right\|^2_2\!\!+\!\!\frac{\beta^2(T)^2}{\gamma\!+\!T}\!\left(\!G^\text{DP}\!+\!\!\frac{1}{T}\!\!\right)\!\!C^2\!+\!\!E_1}\\
			&+\frac{T-1}{T}\frac{1}{\sqrt{\gamma+T}}\\
			&\sqrt{\!\gamma\! \left\|\boldsymbol{w}_1\!\!-\!\boldsymbol{w}^{\text{b},*}\!\right\|^2_2\!\!+\!\!\frac{\beta^2T}{\gamma\!+\!T}\!\left(TG^\text{DP}\!+\!1\!\right)\!C^2\!+E_1}\\
			&\le \frac{3}{\sqrt{\gamma+T}}\\
			&\sqrt{\!\gamma\!\left\|\boldsymbol{w}_1\!-\!\boldsymbol{w}^{\text{b},*}\!\right\|^2_2\!\!+\!\!\frac{\beta^2(T)^2}{\gamma\!+\!T}\!\!\left(\!G^\text{DP}\!+\!\!\frac{1}{T}\!\!\right)\!\!C^2\!+E_1}.
		\end{aligned}
	\end{equation}
	Substituting $E_1$ back to $(\ref{equ: A_2_final})$ yields Lemma \ref{lemma: A5}.
\end{proof}

\subsection{Proof of Theorem 2}
\begin{proof}
	First, we introduce notations for the updates in each iteration.
	The reported gradient of the client $k$ on model $\boldsymbol{w}$ with dataset $\mathcal{D}$ is
	\begin{equation}
		\begin{aligned}
			&\boldsymbol{g}^\text{clip}_{k}(\boldsymbol{w},\! \mathcal{D}) \!=\!\frac{1}{|\mathcal{D}|}\!\sum_{d \in \mathcal{D}}\!\frac{\boldsymbol{g}_k(\boldsymbol{w},d)}{\max\{\!1,\!\oldfrac{\left\| \boldsymbol{g}_k(\boldsymbol{w},d)\right\|_2}{C}\}}.
		\end{aligned}
	\end{equation}
	The aggregated gradient (including virtual steps) at the server is $\hat{\boldsymbol{g}}^\text{clip}_{t}=\tilde{\boldsymbol{g}}^\text{clip}_{t}+ \boldsymbol{n}_{t}$, where
	\begin{align}
		&\tilde{\boldsymbol{g}}^\text{clip}_{t}\!=\!\frac{1}{|\mathcal{S}_{t}|}\!\sum_{k\in \mathcal{S}_{t}}\boldsymbol{g}^\text{clip}_{k}(\boldsymbol{w}_{k,t},\! \mathcal{D}_{k,t}),\\
		&\boldsymbol{n}_{t} = \frac{1}{|\mathcal{S}_{t}|}\sum_{k\in \mathcal{S}_t}\boldsymbol{n}_{k,t}.
	\end{align}
	Here $\boldsymbol{n}_{k,t}$ is the added noise of client $k$ at iteration $t$,
	and $\mathcal{S}_{t}$ is the set of selected clients at iteration $t$ (see Algorithm 1 Line 3).
	In addition, we represent the expectation of the virtual aggregation of clipped gradients as
	\begin{align}
		&\boldsymbol{g}^{\text{clip}}_{t} =\mathbb{E}\left[ \tilde{\boldsymbol{g}}^{\text{clip}}_{t}\right].
	\end{align}
	
	To facilitate the analysis, we further define the following notations.
	When the selection strategy is unbiased selection without clipping, the expected aggregated gradient at $\boldsymbol{w}$ is 
	\begin{align}
		\nabla F(\boldsymbol{w}) = \sum_{k=1}^{N}\frac{|\mathcal{M}_k|}{\sum_{i=1}^{N}|\mathcal{M}_i|}\mathbb{E}\left[\boldsymbol{g}_k(\boldsymbol{w},\mathcal{D}_{k})\right],
	\end{align} 
	where $F(\cdot)$ is the global objective function as in equation (4).
	We use $\nabla F^\text{clip}(\boldsymbol{w})$ to denote the clipped version of $\nabla F(\boldsymbol{w})$, \textit{i.e.},
	\begin{align}
		\nabla F^\text{clip}(\boldsymbol{w}) = \frac{1}{\max\{1,\frac{\|\nabla F(\boldsymbol{w})\|_2}{C}\}}\nabla F(\boldsymbol{w}).
	\end{align}
	
	Then, we begin to prove Theorem 2.
	Based on Assumption 1, we have
	\begin{equation}\label{equ: smooth_descent}
		\begin{aligned}
			F(\tilde{\boldsymbol{w}}_{t+1})-F(\tilde{\boldsymbol{w}}_{t}) \le&
			-\alpha_{t}\nabla F(\tilde{\boldsymbol{w}}_{t})^\top\hat{\boldsymbol{g}}^\text{clip}_{t}+ \frac{1}{2}\alpha_{t}^2L\left\|\hat{\boldsymbol{g}}^\text{clip}_{t}\right\|^2_2.
		\end{aligned}
	\end{equation}
	Taking an expectation over batch, selection, and noise, we have 
	\begin{equation}\label{equ: descent}
		\begin{aligned}
			&\mathbb{E}\left[F(\tilde{\boldsymbol{w}}_{t+1})\right]-F(\tilde{\boldsymbol{w}}_{t})\\ \le& -\alpha_{t}\nabla F(\tilde{\boldsymbol{w}}_{t})^\top\mathbb{E}\left[\hat{\boldsymbol{g}}^\text{clip}_{t}\right] + \frac{1}{2}\alpha_{t}^2L	\mathbb{E}\left[\left\|\hat{\boldsymbol{g}}^\text{clip}_{t}\right\|^2_2\right]\\
			=& \!-\!\alpha_{t}\nabla F(\tilde{\boldsymbol{w}}_{t})\!^\top\!\!\left(\!\nabla F^\text{clip}(\tilde{\boldsymbol{w}}_{t})\!+\! \boldsymbol{g}^\text{clip}_{t}\!-\!\nabla F^\text{clip}(\tilde{\boldsymbol{w}}_{t}\!)\!\right) \\
			&+ \frac{1}{2}\alpha_{t}^2L	\mathbb{E}\left[\left\|\tilde{\boldsymbol{g}}^\text{clip}_{t}+\boldsymbol{n}_{t}\right\|^2_2\right]\\
			\le&\!\! -\!\alpha_{t}\frac{C}{B_1}\!\left\|\nabla\! F(\tilde{\boldsymbol{w}}_{t})\right\|_2^2\!+ \frac{1}{2}\alpha_{t}^2L\left(	C^2+\mathbb{E}\left[\left\|\boldsymbol{n}_{t}\right\|_2^2\right]\right)\\
			&-\!\alpha_{t}\!\nabla\! F(\tilde{\boldsymbol{w}}_{t})^\top\!\left( \boldsymbol{g}^\text{clip}_{t}\!-\!\nabla\! F^\text{clip}(\tilde{\boldsymbol{w}}_{t}) \right).
		\end{aligned}
	\end{equation}
	By rearranging (\ref{equ: descent}), we have
	\begin{equation}\label{equ: 24}
		\begin{aligned}
			\alpha_{t}\frac{C}{B_1}\left\|\nabla F(\tilde{\boldsymbol{w}}_{t})\right\|_2^2
			\le& F(\tilde{\boldsymbol{w}}_{t})-\mathbb{E}\left[F(\tilde{\boldsymbol{w}}_{t+1})\right]\\
			&-\alpha_{t}\nabla F(\tilde{\boldsymbol{w}}_{t})^\top\left( \boldsymbol{g}^\text{clip}_{t}-\nabla F^\text{clip}(\tilde{\boldsymbol{w}}_{t}) \right)\\
			&+ \frac{1}{2}\alpha_{t}^2L\left(	C^2+\mathbb{E}\left[\left\|\boldsymbol{n}_{t}\right\|_2^2\right]\right)
		\end{aligned}
	\end{equation}
	Then, applying Cauchy–Schwarz inequality, we have
	\begin{equation}
		\begin{aligned}\label{equ: 25}
			\alpha_{t}\frac{C}{B_1}\left\|\nabla F(\tilde{\boldsymbol{w}}_{t})\right\|_2^2
			\le& \alpha_{t}\left\|\nabla F(\tilde{\boldsymbol{w}}_{t})\right\|_2\left\| \boldsymbol{g}^\text{clip}_{t}\!-\!\nabla F^\text{clip}(\tilde{\boldsymbol{w}}_{t}) \right\|_2\\
			&+ \frac{1}{2}\alpha_{t}^2L\left(	C^2+\mathbb{E}\left[\left\|\boldsymbol{n}_{t}\right\|_2^2\right]\right).
		\end{aligned}
	\end{equation}
	We bound the term $\left\| \boldsymbol{g}^\text{clip}(\tilde{\boldsymbol{w}}_{t})-\nabla F^\text{clip}(\tilde{\boldsymbol{w}}_{t}) \right\|_2$ in (\ref{equ: 25}) by
	\begin{equation}\label{equ: clip}
		\begin{aligned}
			&\left\|\boldsymbol{g}^\text{clip}_{t}-\nabla F^\textnormal{clip}(\tilde{\boldsymbol{w}}_{t})\right\|_2\\
			\le&\!\left\| \boldsymbol{g}^\text{clip}(\tilde{\boldsymbol{w}}_{t})\!-\!\nabla F^\text{clip}(\tilde{\boldsymbol{w}}_{t}) \right\|_2 + \left\|\boldsymbol{g}_{t}^\textnormal{clip}\!-\!\boldsymbol{g}^\text{clip}(\tilde{\boldsymbol{w}}_{t})\right\|_2\\
			\le&\!
			\left\|\sum_{k=1}^{N}(p^{\text{s}}_{k}\!-\!p^{\text{u}}_{k})\boldsymbol{g}^\text{clip}(\tilde{\boldsymbol{w}}_{t})\right\|_2 \!\!+\! \left\|\boldsymbol{g}^\text{clip,unbiased}(\tilde{\boldsymbol{w}}_{t})\! -\! \nabla F^\text{clip}(\tilde{\boldsymbol{w}}_{t})\right\|_2\\
			&+ B_2\\
			\le& \left\|\boldsymbol{p}^\text{s}-\boldsymbol{p}^\text{u}\right\|_1C+\left\|\Delta_t\right\|_1C + B_2\\
			\le& G^\text{select, clip}C.
		\end{aligned}
	\end{equation}
	Substitute (\ref{equ: clip}) into (\ref{equ: 25}) yields
	\begin{equation}
		\begin{aligned}
			\alpha_{t}\frac{C}{B_1}\left\|\nabla F(\tilde{\boldsymbol{w}}_{t})\right\|_2^2 \le& F(\tilde{\boldsymbol{w}}_{t})-\mathbb{E}\left[F(\tilde{\boldsymbol{w}}_{t+1})\right] \\
			&+\alpha_{t}\left\|\nabla F(\tilde{\boldsymbol{w}}_{t})\right\|_2G^\text{select,clip}C \\
			&+ \frac{1}{2}\alpha_{t}^2L\!\left(\!	C^2\!+\!\mathbb{E}\left[\left\|\boldsymbol{n}_{t}\right\|_2^2\right]\right).
		\end{aligned}
	\end{equation}
	Taking expectation over all iterations and considering $\mathbb{E}\left[F(\boldsymbol{w}_{T+1})\right] \ge F(\boldsymbol{w}^*)$, we have
	\begin{equation}
		\begin{aligned}		&\frac{1}{T}\!\sum_{t=1}^{T}\alpha_{t}\frac{C}{B_1}\mathbb{E}\left[\left\|\nabla F(\tilde{\boldsymbol{w}}_{t})\right\|_2^2\right]\\
			\le& \frac{F(\boldsymbol{w}_1)\!-\!F(\boldsymbol{w}^*)}{T} \!+\!\frac{1}{T}\sum_{t=1}^{T}\alpha_{t}\mathbb{E}\left[\left\|\nabla F(\tilde{\boldsymbol{w}}_{t})\right\|_2\right]G^\text{select,clip}C\\
			&+ \frac{1}{T}\sum_{t=1}^{T}\frac{1}{2}\alpha_{t}^2L\left(	C^2+\mathbb{E}\left[\left\|\boldsymbol{n}_{t}\right\|_2^2\right]\right).
		\end{aligned}
	\end{equation}
	Setting the learning rate $\alpha_{t} = \oldfrac{B_1}{C} \oldfrac{1}{\sqrt{T}}$ yields
	\begin{equation}\label{equ: quadratic}
		\begin{aligned}
			&\frac{1}{T}\sum_{t=1}^{T}\mathbb{E}\left[\left\|\nabla F(\tilde{\boldsymbol{w}}_{t})\right\|_2^2\right]\\
			\le& \frac{F(\boldsymbol{w}_1)-F(\boldsymbol{w}^*)}{\sqrt{T}} +\frac{1}{T}\sum_{t=1}^{T}\mathbb{E}\left[\left\|\nabla F(\tilde{\boldsymbol{w}}_{t})\right\|_2\right]G^\text{select,clip}B_1\\
			&+ \frac{1}{T}\sum_{t=1}^{T}\frac{1}{2}\oldfrac{B_1}{C} \oldfrac{1}{\sqrt{T}}L\left(C^2+\mathbb{E}\left[\left\|\boldsymbol{n}_{t}\right\|_2^2\right]\right)\\
			\le& \frac{F(\boldsymbol{w}_1)\!-\!F(\boldsymbol{w}^*)}{\sqrt{T}}\!+\!\!\sqrt{\frac{1}{T}\sum_{t=1}^{T}\mathbb{E}\!\left[\left\|\nabla F(\tilde{\boldsymbol{w}}_t)\right\|_2^2\right]}\!G^\text{select,clip}B_1\\
			&+ \frac{1}{T}\sum_{t=1}^{T}\frac{1}{2}\oldfrac{B_1^2}{C^2} \oldfrac{1}{\sqrt{T}}L\left(	C^2+\mathbb{E}\left[\left\|\boldsymbol{n}_{t}\right\|_2^2\right]\right).
		\end{aligned}
	\end{equation}
	Inequality (\ref{equ: quadratic}) is an quadratic inequality with respect to $\sqrt{\frac{1}{T}\sum_{t=1}^{T}\mathbb{E}\left[\left\|\nabla F(\tilde{\boldsymbol{w}}_t)\right\|_2^2\right]}$. 
	Substituting the noise expression and solving the inequality, we have
	\begin{equation}
		\begin{aligned}
			&\sqrt{\frac{1}{T}\sum_{t=1}^{T}\mathbb{E}\left[\left\|\nabla F(\tilde{\boldsymbol{w}}_{t})\right\|_2^2\right]}\\\le& \frac{B_1}{2}\!\left[G^\text{select,clip}+\!\sqrt{\!\left(\!G^\text{select,clip}\!\right)^2 \!+\! 2\!\sum_{k=1}^{N}(p^\text{s}_{k})^2D\sqrt{T}LV_k}\!\right]\\
			&+\sqrt{\frac{F(\boldsymbol{w}_1 )-F(\boldsymbol{w}^*)}{\sqrt{T}}+\frac{1}{2} \oldfrac{1}{\sqrt{T}}LB_1^2}.
		\end{aligned}
	\end{equation}
	This completes the proof.
\end{proof}

\subsection{Proof of Proposition 1} \label{sec: proof_convex}
The constraints in Problem 1 are linear, so it suffices to demonstrate that the objective function is convex.
We will prove the convexity based on the following lemma.
\begin{lemma}\label{lemma: convex}
	The square root of the sum of two squared convex and non-negative functions $f_1$ and $f_2$ is a convex function.
\end{lemma}
\begin{proof}
	Suppose we have a function 
	\begin{align}
		f(X) = \sqrt{f_1^2(X)+f_2^2(X)}, 
	\end{align}
	which is the square root of the sum of two squared convex and non-negative functions $f_1(X)$ and $f_2(X)$.
	We will prove the convexity by proving
	\begin{align}\label{equ: convexity}
		f(\alpha X_1 +(1-\alpha)X_2)\le \alpha f(X_1)+(1-\alpha) f(X_2).
	\end{align}
	
	First, for the left-hand side of (\ref{equ: convexity}), we have
	\begin{equation}\label{equ: convex part1}
		\begin{aligned}
			&f(\alpha X_1+(1-\alpha)X_2) \\
			=& \sqrt{f_1^2(\alpha X_1+(1-\alpha)X_2)+f_2^2(\alpha X_1+(1-\alpha)X_2)}.
		\end{aligned}
	\end{equation}
	Taking square power of both sizes of (\ref{equ: convex part1}) yields
	\begin{equation}
		\begin{aligned}
			&f(\alpha X_1+(1-\alpha)X_2)^2 = \\
			&f_1^2(\alpha
			X_1+(1-\alpha)X_2)+f_2^2(\alpha X_1+(1-\alpha)X_2).
		\end{aligned}
	\end{equation}
	By convexity of $f_1$ and $f_2$, we have	
	\begin{equation}\label{equ: square convex part1}
		\begin{aligned}
			&f(\alpha X_1+(1-\alpha)X_2)^2 \le\\
			&[\alpha f_1(X_1)+(1-\alpha)f_1(X_2)]^2\!+\![\alpha f_2(X_1)\!+\!(1-\alpha)f_2(X_2)]^2.
		\end{aligned}
	\end{equation}
	Furthermore, for the right-hand side of (\ref{equ: convexity}), we have
	\begin{equation}
		\begin{aligned}\label{equ: convex part2}
			&\alpha f(X_1)+(1-\alpha)f(X_2) =\\ &\alpha\sqrt{f_1^2(X_1)+f_2^2(X_1)}+(1-\alpha)\sqrt{f_1^2(X_2)+f_2^2(X_2)}.
		\end{aligned}
	\end{equation}
	Taking square power of both sizes of (\ref{equ: convex part2}), we have
	\begin{equation}\label{equ: square convex part2}
		\begin{aligned}
			&[\alpha f(X_1)+(1-\alpha)f(X_2)]^2\\
			&= \alpha^2f_1^2(X_1)\!+\!\alpha^2f_2^2(X_1)\!+\!(1\!-\!\alpha)^2f_1^2(X_2)\!+\!(1\!-\!\alpha)^2f_2^2(X_2)\\
			&+2\alpha(1-\alpha)\sqrt{f_1^2(X_1)+f_2^2(X_1)}\sqrt{f_1^2(X_2)+f_2^2(X_2)}.
		\end{aligned}
	\end{equation}
	Then, we can express the difference of the left-hand sides of (\ref{equ: square convex part1}) and (\ref{equ: square convex part2}) as
	\begin{equation}
		\begin{aligned}
			&f(\alpha X_1+(1-\alpha)X_2)^2 - [\alpha f(X_1)+(1-\alpha)f(X_2)]^2\\ 
			&\le 2\alpha(1-\alpha)\left[f_1(X_1)f_1(X_2)+f_2(X_1)f_2(X_2)
			\right.\\
			&\left.-\sqrt{f_1^2(X_1)+f_2^2(X_1)}\sqrt{f_1^2(X_2)+f_2^2(X_2)}\right].
		\end{aligned}
	\end{equation}
	Let
	\begin{align}
		A =& f_1(X_1)f_1(X_2)+f_2(X_1)f_2(X_2),\\
		B=& \sqrt{f_1^2(X_1)+f_2^2(X_1)}\sqrt{f_1^2(X_2)+f_2^2(X_2)}.
	\end{align}
	For $A^2$ and $B^2$, we have
	\begin{equation}
		\begin{aligned}
			A^2=&f_1^2(X_1)f_1^2(X_2)+f_2^2(X_1)f_2^2(X_2)\\
			&+2f_1(X_1)f_2(X_1)f_1(X_2)f_2(X_2),\\
		\end{aligned}
	\end{equation}
	\begin{equation}
		\begin{aligned}
			B^2 =&f_1^2(X_1)f_1^2(X_2)+f_1^2(X_1)f_2^2(X_2)+f_2^2(X_1)f_1^2(X_2)\\
			&+f_2^2(X_1)f_2^2(X_2).
		\end{aligned}
	\end{equation}
	So, by inequality of arithmetic and geometric means, we can conclude that $A^2\le B^2$.
	Because $f_1$ and $f_2$ are both non-negative, we have $A \le B$.
	Finally, $f(\alpha X_1 +(1-\alpha)X_2)\le \alpha f(X_1)+(1-\alpha) f(X_2)$ holds.
\end{proof}
With Lemma \ref{lemma: convex}, we can prove Proposition 1 with the following.
\begin{proof}
	For the objective in Problem 1, define:
	\begin{align}
		f_1 =& \left\|\boldsymbol{p}^\text{s}-\boldsymbol{p}^\text{u}\right\|_1,\\
		f_2 =& \sqrt{\eta\sum_{k=1}^{N}(p^\text{s}_{k})^2DV_k}.
	\end{align}
	It can be verified that $f_1$ is a 1-norm and $f_2$ is a 2-norm of a linearly transformed probability vector.
	They are both convex and non-negative.
	Therefore, $\sqrt{f_1^2+f_2^2}$ is still convex by Lemma \ref{lemma: convex}.
	As $f_1$ is convex and the sum of convex functions are still a convex function. 
	We conclude that $f_1 + \sqrt{f_1^2+f_2^2}$ is convex.
	This completes the proof of Proposition 1.
\end{proof}
\subsection{Proof of Proposition 2}
\begin{proof}
	The KKT conditions of Problem 1 require that 
	\begin{align}\label{equ: KKT1}
		&\frac{\partial F^\text{obj}}{\partial p^\text{s}_{k}} + \lambda - \mu_k = 0,\\
		& \mu_k \ge 0,\\
		& \mu_kp^\text{s}_{k} =0,\\
		&\sum_{k=1}^{N}p^\text{s}_{k} = 1,\\
		&p^\text{s}_{k}\ge0,
	\end{align}
	where $F^\text{obj}$ is the objective function, $\lambda$ is the Lagrange multiplier for the equality constraint, and $\mu_k$ are the multipliers for the non-negativity constraints.
	
	Expanding the left hand side of (\ref{equ: KKT1}) for a given user $k$ yields
	\begin{equation}\label{equ: kkt}
		\begin{aligned}
			&\frac{\partial F^\text{obj}}{\partial p^\text{s}_{k}} + \lambda - \mu_k = \frac{d |p^\text{s}_{k}\!-\! p^\text{u}_{k}|}{d p^\text{s}_{k}}\\ &+\frac{2\left\|\boldsymbol{p}^\text{s}-\boldsymbol{p}^\text{u}\right\|_1\frac{d |p^\text{s}_{k} - p^\text{u}_{k}|}{d p^\text{s}_{k}}+2\eta p^\text{s}_{k}DV_k}{2\sqrt{\left\|\boldsymbol{p}^\text{s}-\boldsymbol{p}^\text{u}\right\|^2_1+\eta\sum_{i=1}^{N} (p^\text{s}_{i})^2DV_i}}\!+\!\lambda -\mu_k=0.
		\end{aligned}
	\end{equation}
	We next prove Proposition 2 by contradiction. 
	Suppose that $p^\text{s}_{k}=0$ for some $k\in\mathcal{N}$.
	The partial derivative at $p^\textbf{s}_k=0$ yields 
	\begin{align}
		\frac{d |p^\text{s}_{k}-p^\text{u}_{k}|}{d p^\text{s}_{k}} = -1,
	\end{align}	
	leading to
	\begin{equation}\label{equ: kkt 1}
		\begin{aligned}
			-1+\frac{-2\left\|\boldsymbol{p}^\text{s}-\boldsymbol{p}^\text{u}\right\|_1}{2\sqrt{\left\|\boldsymbol{p}^\text{s}-\boldsymbol{p}^\text{u}\right\|^2_1+\eta\sum_{i=1}^{N} (p^\text{s}_{i})^2DV_i}}-\mu_k=-\lambda <0.
		\end{aligned}
	\end{equation}
	Since $\sum_{i=1}^{N}p^\text{s}_{i} = \sum_{i=1}^{N}p^\text{u}_{i} = 1$, there must exist some $j\neq k$ such that $p^\text{s}_{j}>p^\text{u}_{j}$. 
	
	For this $j$, we have
	\begin{equation}\label{equ: kkt 2}
		\begin{aligned}
			1+\frac{2\left\|\boldsymbol{p}^\text{s}-\boldsymbol{p}^\text{u}\right\|_1+2\eta p^\text{s}_{j}DV_j}{2\sqrt{\left\|\boldsymbol{p}^\text{s}-\boldsymbol{p}^\text{u}\right\|^2_1+\eta\sum_{i=1}^{N} (p^\text{s}_{i})^2DV_i}}=-\lambda >0,
		\end{aligned}
	\end{equation}
	which contradicts $-\lambda < 0$.
	Thus, $p_{k}^\text{s}>0$, for all $k\in\mathcal{N}$.
	This completes the proof.
\end{proof}

\end{document}